\relax

\documentclass[11pt]{article}

\usepackage{fullpage,amssymb,amsmath,url}
\usepackage{makeidx,mdwtab}  
\usepackage{placeins}
\usepackage{graphicx,upgreek,morenotations,rotating}
\usepackage{subfigure}

\usepackage{diagbox}
\usepackage{booktabs}

\newcommand{\trot}{\textsc{trot}}
\newcommand{\sotrot}{\textsc{so}--\trot}
\newcommand{\kltrot}{\textsc{kl}--\trot}
\newcommand{\sm}{\textsc{sm}}

\newcommand{\ot}{\textsc{ot}}
\newcommand{\kl}{\textsc{kl}}
\newcommand{\vectorized}[1]{\bm{vec}(#1)}

\usepackage{algorithm}
\usepackage{algorithmic}
\usepackage{bbm}

\usepackage{color, colortbl}
\definecolor{Gray}{gray}{0.9}

\author{
  Boris Muzellec\\
Ecole Polytechnique\\
  \texttt{boris.muzellec@polytechnique.edu}
\and
  Richard Nock\\
Data61, the Australian National University \& the University of Sydney\\
  \texttt{richard.nock@data61.csiro.au}
\and
  Giorgio Patrini\\
The Australian National University \& Data61\\
  \texttt{giorgio.patrini@anu.edu.au}
\and
  Frank Nielsen\\
Ecole Polytechnique \& Sony CS Labs, Inc.\\
  \texttt{Frank.Nielsen@acm.org}
}

\date{}
\begin{document}
\title{Tsallis Regularized Optimal Transport and Ecological Inference}

\maketitle



\begin{abstract}
Optimal transport is a powerful framework for computing distances between probability distributions. We unify the two main approaches to optimal transport, namely Monge-Kantorovitch and Sinkhorn-Cuturi, into what we define as Tsallis regularized optimal transport (\trot). \trot~interpolates a rich family of distortions from Wasserstein to Kullback-Leibler, encompassing as well Pearson, Neyman and Hellinger divergences, to name a few. We show that metric properties known for Sinkhorn-Cuturi generalize to \trot, and provide efficient algorithms for finding the optimal transportation plan with formal convergence proofs. We also present the first application of optimal transport to the problem of ecological inference, that is, the reconstruction of joint distributions from their marginals, a problem of large interest in the social sciences. \trot~provides a convenient framework for ecological inference by allowing to compute the joint distribution --- that is, the optimal transportation plan itself --- when side information is available, which is \textit{e.g.} typically what census represents in political science. Experiments on data from the 2012 US presidential elections display the potential of \trot~in delivering a faithful reconstruction of the joint distribution of ethnic groups and voter preferences.
\end{abstract}

\section{Introduction}

Optimal transport (\ot) allows to compare probability distributions by exploiting the underlying metric space on their supports \cite{kOT,mMS}. A number of prominent applications allow for a natural definition of this underlying metric space, from image processing \cite{rtgTE} to natural language processing \cite{kskwFW}, music processing \cite{ffceOS} and computer graphics \cite{sdpcbndgCW}. 

One key problem of \ot~is its processing complexity --- cubic in the support size, ignoring low order terms (on state of the art LP solvers \cite{cSD}). Moreover, the optimal transportation plan has often many zeroes, which is not desirable in some applications.
An important workaround was found and consists in penalizing the transport cost with a Shannon entropic regularizer \cite{cSD}. At the price of changing the transport distance, for a distortion with metric related properties, comes an algorithm with geometric convergence rates \cite{cSD,flOT}. As a result, we can picture two separate approches to \ot: one essentially relies on the initial Monge-Kantorovitch formulation optimizing the transportation cost itself \cite{vOT}, but is computationally expensive; the other is based on tweaking the transportation cost by Shannon regularizer \cite{cSD}. The corresponding optimization algorithm, grounded in a variety of different works \cite{cAGI,sDE,sTR}, is fast and can be very efficiently parallelized \cite{cSD}. 

Our paper brings \textit{three} contributions. (i) We interpolate these two worlds using a family of entropies celebrated in nonextensive statistical mechanics, Tsallis entropies \cite{tPG}, and hence we define the Tsallis regularized optimal transport (\trot). We show that the metric properties for Shannon entropy still hold in this more general case, and prove new properties that are key to our application. (ii) We provide efficient optimization algorithms to compute \trot~and the optimal transportation plan. (iii) Last but not least, we provide a new application of \trot~to a field in which this optimal transportation plan is the key unknown: the problem of ecological inference. 

\begin{figure}[t]
\begin{center}
\includegraphics[trim=160bp 180bp 200bp 120bp,clip,width=0.8\columnwidth]{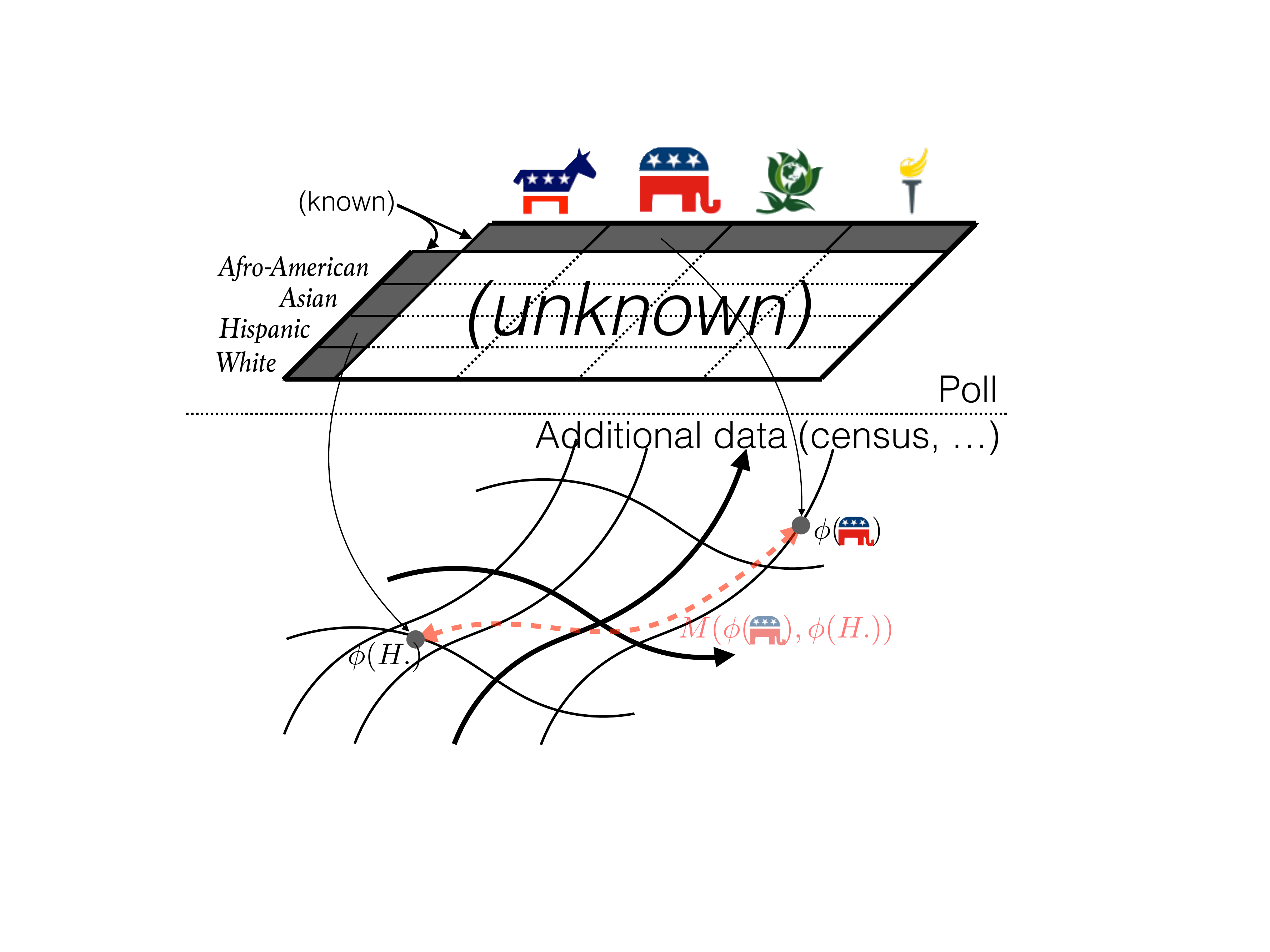}
\end{center}
\caption{Top: suppose we know (in grey) marginals for the US presidential election (topmost row) and ethnic breakdowns in the US population (leftmost column). Can we recover an estimated joint distribution (white cells) ? If side information is available such as individual level census data (bottom, as depicted on a Hilbert manifold with $\phi$-coordinates), then distances can be computed within the supports (dashed red), and optimal transport can provide an estimation of the joint distribution.}
\label{f-sch}
\end{figure}

Ecological inference deals with recovering information from aggregate data. It arises in a diversity of applied fields such as econometrics \cite{cmRS,tcmCL},  sociology and political science \cite{kAST,ktrEI} and epidemiology \cite{wsHE}, with a long history \cite{rEC}; interestingly, the empirical software engineering community has also explored the idea \cite{pfdEI}. Its iconic application is inferring electorate behaviour: given turnout results for several parties and proportions of some population strata, \textit{e.g.} percentages of ethnic groups, for many geographical regions such as counties, the aim is to recover contingency tables for parties $\times$ groups for all those counties. In the language of probability the problem is isomorphic to the following: given two random variables and their respective marginal distributions --- conditioned to another variable, the geography ---, compute their conditional joint distribution (See Figure \ref{f-sch}).

The problem is fundamentally under-determined and any solution can only either provide loose deterministic bounds \cite{ddAA,cmRS,tcmCL} or needs to enforce additional assumptions and prior knowledge on the data domain \cite{kAST}. More recently, the problem has witnessed a period of renaissance along with the publication of a diversity of methods from the second family, mostly inspired by distributional assumptions as summarised in \cite{ktrEI}. Closer to our approach, \cite{jmcAI} follows the road of a minimal subset of assumptions and frame the inference as an optimization problem. The method favors one solution according to some information-theoretic solution, \textit{e.g.} the Cressie-Read power divergence, intended as an entropic measure of the joint distribution.

There is an intriguing link between optimal transport and ecological inference: if we can figure out the computation of the ground metric, then the optimal transportation plan provides a solution to the ecological inference problem. This is appealing because it ties the computation of the joint distribution to a ground individual distance between people. Figure \ref{f-sch} gives an example. As recently advocated in ecological inference \cite{fwsWS}, it turns out that we have access to more and more side information that helps to solve ecological inference --- in our case, the computation of this ground metric. Polls, census, social networks are as many sources of public or private data that can be of help. It is not our objective to show how to best compute the ground metric, but we show an example on real world data for which a simple approach gives very convincing results.

To our knowledge, there is no former application of optimal transport (regularized or not) to ecological inference. The closest works either assume that the joint distribution follows a random distribution constrained to structural or marginal constraints \cite{fmTB} (and references therein) or modify the constraints to the marginals and / or add constraints to the problem \cite{dmvSO}. In all cases, there is no ground metric (or anything that looks like a cost) among supports that ties the computation of the joint distribution. More importantly, as noted in \cite{fwsWS}, traditional ecological inference would not use side information of the kind that would be useful to estimate our ground metric.

This paper is organized as follows. In Section $\S$ \ref{sec-def}, we present the main definitions for \ot. $\S$ \ref{sec-trot} presents \trot~and its geometric properties. $\S$ \ref{sec-soto} presents the algorithms to compute \trot~and the optimal transportation plan, and their properties. $\S$ \ref{sec-exp} details experiments. A last Section concludes with open problems. \textit{All proofs}, related comments, and some experiments are deferred to a Supplementary Material (\sm).

\section{Basic definitions and concepts}\label{sec-def}


In the following, we let $\bigtriangleup_n \defeq \lbrace \ve{x} \in \mathbb{R}_+^n : \ve{x}^\top\ve{1} = 1\rbrace$ denote the probability simplex (bold faces like $\ve{x}$ denote vectors). 
$\inner{P}{Q}\defeq \vectorized{P}^\top \vectorized{Q}$ denotes Frobenius product ($\vectorized{.}$ is the vectorization of a matrix). 
For any two $\ve{r},\ve{c}\in\bigtriangleup_n$, we define their \textit{transportation polytope} $U(\ve{r},\ve{c}) \defeq \lbrace P \in \mathbb{R}_+^{n\times n} : P\ve{1} = \ve{r}, P^\top\ve{1} = \ve{c}\rbrace$.
For any cost matrix $M\in \mathbb{R}^{n\times n}$, the \textit{transportation distance} between $\ve{r}$ and $\ve{c}$ as the solution of the following minimization problem:
\begin{eqnarray}\label{eq:OT}
d_M(\ve{r},\ve{c}) & \defeq & \min_{P \in U(\ve{r},\ve{c})} \inner{P}{M}\:\:.
\end{eqnarray}
Its argument, $P^\star\defeq \arg\min_{P \in U(\ve{r},\ve{c})} \inner{P}{M}$ 
is the \textit{(optimal) transportation plan} between $\ve{r}$ and $\ve{c}$. Assuming $M \neq 0$, $P^\star$ is unique. Furthermore, if $M$ is a \textit{metric matrix}, then $d_M$ is also a metric \cite[\S 6.1]{vOT}. 

In current applications of optimal transport, the key unknown is usually the distance $d_M$ \cite{cSD,cdFC,gcpbSO,qhchlNN,sdpcbndgCW} (etc). In the context of ecological inference \cite{jmcAI}, it is rather $P^\star$: $P^\star$ describes a joint distribution between two discrete random variables $\RR$ and $\CC$ with respective marginals $\ve{r}$ and $\ve{c}$, $p^\star_{ij} = \Pr(\RR = r_i\wedge \CC = c_j)$, for example the support of $\RR$ being the votes for year $Y$ US presidential election, and $\CC$ being the ethnic breakdown in the US population in year $Y$, see Figure \ref{f-sch}. In this case, $p^\star_{ij}$ denotes an "ideal" joint distribution of votes within ethnicities, ideal in the sense that it minimizes a distance based on the belief that votes \textit{correlate} positively with a similarity between an ethnic profile and a party's profile. While we will carry out most of our theory on formal transportation grounds, requiring in particular that $M$ be a distance matrix, it should be understood that requiring just "correlation" alleviates the need for $M$ to formally be a distance for ecological inference.

\section{Tsallis Regularized Optimal Transport}\label{sec-trot}
\begin{figure}[t]
\begin{center}
\includegraphics[trim=50bp 50bp 50bp 30bp,clip,width=0.6\columnwidth]{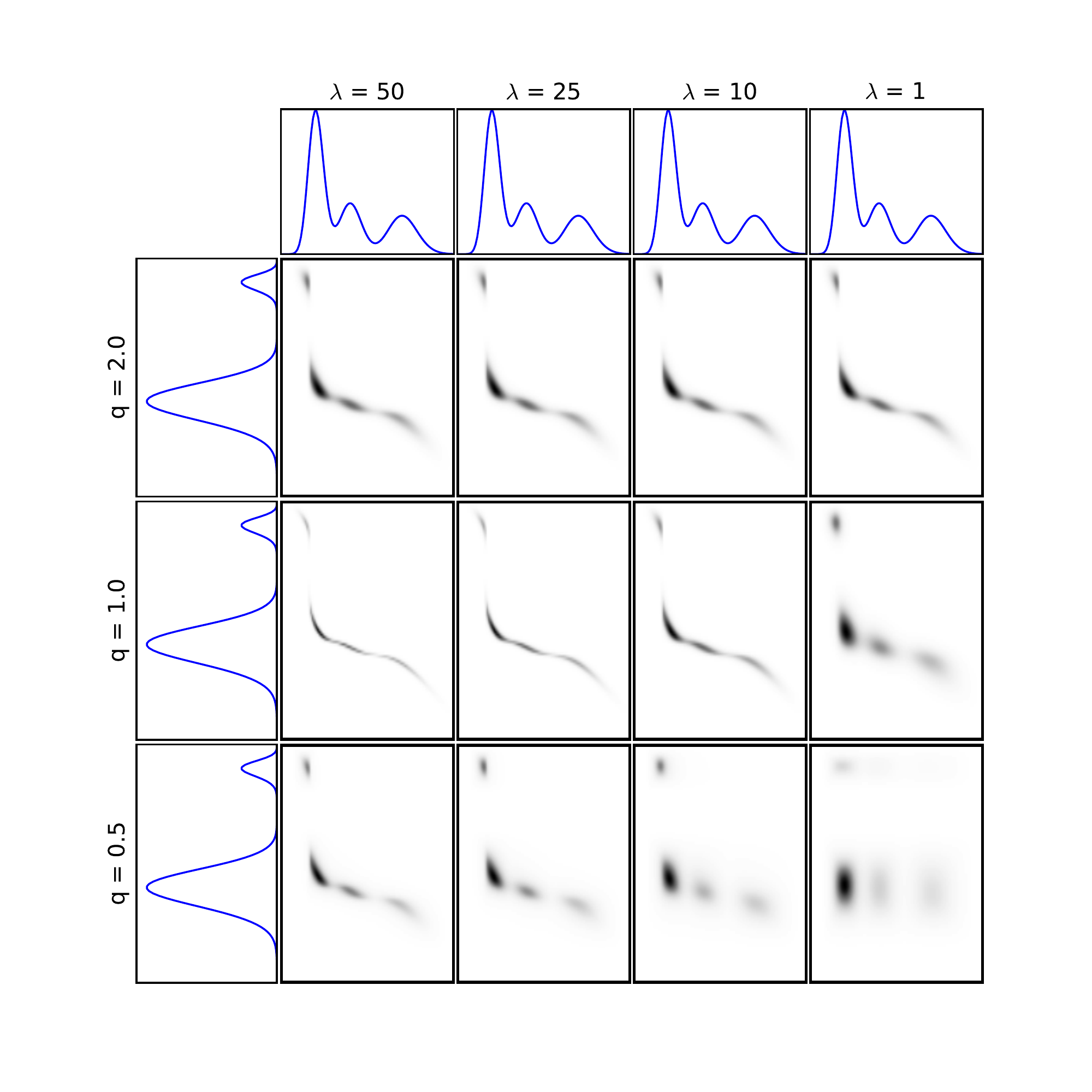}
\end{center}
\caption{Example of optimal \trot~transportation plans (grey levels) for two marginals (blue), with different values of $q$ (in $K_{1/q}$, \textit{Cf} Lemma \ref{lemd}) that corresponds to square Hellinger, Kullback-Leibler and Pearson's $\chi^2$ divergence (top to bottom, conventions follow \cite{sdpcbndgCW}).}
\label{f-exp1}
\end{figure}

For any $\ve{p}\in \mathbb{R}_+^{n}, q\in \mathbb{R}$, the \textit{Tsallis entropy} of $\ve{p}$, $H_q(\ve{p})$ is:
\begin{eqnarray}
H_q(\ve{p}) & \defeq & \frac{1}{1-q}\cdot\sum_i (p_i^q - p_i)\:\:,
\end{eqnarray}
and for any $P \in \mathbb{R}_+^{n\times n}$, we let $H_q(P) \defeq H_q(\vectorized{P})$. Notably, we have $\lim_{q\rightarrow 1} H_q(\ve{p}) = -\sum_i p_i \ln p_i \defeq H_1(\ve{p})$, which is just Shannon's entropy. For any $\lambda > 0$, we define the Tsallis Regularized Optimal Transport (\trot) distance.
\begin{definition}\label{defTROT}
The \trot($q,\lambda,M$)~distance (or \trot~distance for short) between $\ve{r}$ and $\ve{c}$ is:
\begin{eqnarray}
d^{\lambda,q}_M (\ve{r},\ve{c}) & \defeq & \min_{P \in U(\ve{r},\ve{c})} \inner{P}{M} - \frac{1}{\lambda}\cdot H_q(P)\:\:.\label{defMin}
\end{eqnarray}
\end{definition}
A simple yet important property is that \trot~distance unifies both usual modalities of optimal transport. It generalizes optimal transport (\ot) when $q\rightarrow 0$, since $H_q$ converges to a constant and so the \ot-distance is obtained up to a constant additive term \cite{kOT,mMS}. It also generalizes the regularized optimal transport approach of \cite{cSD} since $\lim_{q\rightarrow 1} d^{\lambda,q}_M (\ve{r},\ve{c}) =  d^{\lambda}_M (\ve{r},\ve{c})$, the Sinkhorn distance between $\ve{r}$ and $\ve{c}$ \cite{cSD}. There are several important structural properties of $d^{\lambda,q}_M$ that motivate the unification of both approaches. To state them, we respectively define the $q$-logarithm, 
\begin{eqnarray}
\log_q(x) & \defeq & (1-q)^{-1}\cdot(x^{1-q}-1)\:\:,
\end{eqnarray} 
the $q$-exponential, $\exp_q(x) \defeq (1+(1-q)\cdot x)^{1/(1-q)}$ and Tsallis relative $q$-entropy between $P,R\in \mathbb{R}_+^{n\times n}$ as:
\begin{eqnarray}
K_q(P , R) \hspace{-0.3cm}& \defeq & \hspace{-0.3cm}\frac{1}{1-q}\cdot\sum_{i,j} \left(qp_{ij} + (1-q)r_{ij}-p_{ij}^qr_{ij}^{1-q}\right) \:\:.
\end{eqnarray}
Taking joint distribution matrices $P,R$ and $q\rightarrow 1$ allows to recover the natural logarithm, the exponential and Kullback-Leibler (\kl) divergence, respectively \cite{aIG}. Other notable examples include (i) Pearson's $\chi^2$ statistic ($q=2$), (ii) Neyman's statistic ($q=-1$), (iii) square Hellinger distance ($q=1/2$) and the reverse \kl~divergence if scaled appropriately by $q$ \cite{jmcAI}, which also allows to span Amari's $\alpha$ divergences for $\alpha = 1 - 2q$ \cite{aIG}. For any function $f:{\mathbb{R}}\rightarrow {\mathbb{R}}$, denoting $f(P)$ for matrix $P$ as the matrix whose general term is $f(p_{ij})$.
\begin{lemma}\label{lemd}
Let $\tilde{U}\defeq \exp_q(-1)\exp_q^{-1}(\lambda M)$. Then:
\begin{eqnarray}
d^{\lambda,q}_M (\ve{r},\ve{c}) & = & \frac{1}{\lambda}\cdot \min_{P \in U(\ve{r},\ve{c})} K_{1/q}(P^q , \tilde{U}^q) + g(M)\:\:,\label{tsent}
\end{eqnarray}
where $g(M) \defeq (1/\lambda)\cdot \inner{\tilde{U}^q}{1}$ does not play any role in the minimization of $K_{1/q}(.\|.)$.
\end{lemma}
Lemma \ref{lemd} shows that the \trot~distance is a divergence involving \textit{escort} distributions \cite[$\S$ 4]{aIG}, a particularity that disappears in Sinkhorn distances since it becomes an ordinary \kl~divergence between distributions. 
Predictably, the generalization is useful to create new solutions to the regularized optimal transport problem that are not captured by Sinkhorn distances (\textit{solution} refers to (optimal) transportation plans, \textit{i.e.} the argument of the $\min$ in eq. (\ref{defMin})).
\begin{theorem}\label{thDist}
Let ${\mathcal{S}}_{\lambda,q}(\ve{r},\ve{c})$ denote the set of solutions of eq. (\ref{defMin}) when $M$ ranges over all distance matrices. Then $\forall q, q'$ such that $q\neq q'$, $\forall \lambda, \lambda'$, ${\mathcal{S}}_{\lambda,q}(\ve{r},\ve{c})\neq {\mathcal{S}}_{\lambda',q'}(\ve{r},\ve{c})$.
\end{theorem}
Figure \ref{f-exp1} provides examples of solutions. Adding the free parameter $q$ is not just interesting for the reason that we bring new solutions to the table: $(1/q)\cdot K_q(\ve{p},\ve{r})$ turns out to be Cressie-Read Power Divergence (for $q = \lambda+1$, \cite{jmcAI}), and so \trot~has an applicability in ecological inference that Sinkhorn distances alone do not have. In addition, we also generalize two key facts already known for Sinkhorn distances \cite{cSD}. First, the solution to \trot~is unique (for $q\neq 0$) and satisfies a simple analytical expression amenable to convenient optimization.
\begin{theorem}\label{thUnique}
There exists exactly one matrix $P \in U(\ve{r},\ve{c})$ solution to \trot($q,\lambda,M$). It satisfies:
\begin{eqnarray}\label{eq:KKT}
p_{ij} & = & \exp_q(-1)\exp_q^{-1}(\alpha_i + \lambda m_{ij} + \beta_j)\:\:,\forall i,j\:\:.
\end{eqnarray} 
($\ve{\alpha},\ve{\beta} \in \mathbb{R}^n$ are unique up to an additive constant).
\end{theorem}
Second, we can tweak \trot~to meet distance axioms. Let
\begin{eqnarray}
d_{M, \alpha, q}(\ve{r},\ve{c}) & \defeq & \min_{\substack{P \in U(\ve{r},\ve{c})\\ H_q(P) - H_q(\ve{r}) - H_q(\ve{c}) \geq \alpha}} \inner{P}{M}\:\:,\label{trot22}
\end{eqnarray}
where $\alpha \geq 0$. For any $M, \ve{r}, \ve{c}, \lambda\geq 0$, $\exists \alpha\geq 0$ such that $d_{M,\alpha, q}(\ve{r},\ve{c}) = d^{\lambda,q}_M (\ve{r},\ve{c})$. Also, the following holds.
\begin{theorem}\label{thAxiom}
For $q \geq 1,\alpha \geq 0$ and if $M$ is a metric matrix, function $(\ve{r}, \ve{c}) \rightarrow \mathbbm{1}_{\lbrace \ve{r}=\ve{c}\rbrace} d_{M, \alpha, q}(\ve{r},\ve{c})$ is a distance.
\end{theorem}
Theorem \ref{thAxiom} is a generalization of \cite[Theorem 1]{cSD} (for $q=1$). As we explain more precisely in \sm~(Section \ref{proof_thAxiom}), there is a downside to using $d_{M, \alpha, q}$ as proof of the good properties of $d^{\lambda,q}_M$: the triangle inequality, key to Euclidean geometry, transfers to $d^{\lambda,q}_M$ \textit{with varying and uncontrolled parameters} --- in the inequality, the three values of $\lambda$ may all be different! This does not break down the good properties of $d^{\lambda,q}_M$, it just calls for workarounds. We now give one, which replaces $d_{M, \alpha, q}$ by the quantity ($\beta \in \mathbb{R}$ is a constant):
\begin{eqnarray}
d^{\lambda,q,\beta}_M (\ve{r},\ve{c}) & \defeq &  d^{\lambda,q}_M (\ve{r},\ve{c}) + \frac{\beta}{\lambda}\cdot \left(H_q(\ve{r}) + H_q(\ve{c})\right)\:\:.\label{trot3}
\end{eqnarray}
This has another trivial advantage that $d_{M, \alpha, q}$ does not have: the solutions (optimal transportation plans) are always the \textit{same} on both sides. Also, the right-hand side is lowerbounded for any $\ve{r}, \ve{c}$ and the trick that ensures the identity of the indiscernibles still works on $d^{\lambda,q,\beta}_M$. The good news is that if $q=1$, $d^{\lambda,q,\beta}_M$, \textit{as is}, can satisfy the triangle inequality.
\begin{theorem}\label{thdist3}
$d^{\lambda,1,\beta}_M$ satisfies the triangle inequality, $\forall \beta \geq 1$.
\end{theorem}
Hence, the solutions to $d^{\lambda,1}_M$ \textit{are} optimal transport plans for distortions that meet the triangle inequality. This is new compared to \cite{cSD}. For a general $q\geq 1$, the proof, in Supplementary Material (Section \ref{proof_thAxiom}), shows more, namely that $d^{\lambda,q,1/2}_M$ satisfies a weak form of the identity of the indiscernibles. Finally, there always exist a value $\beta\geq 0$ such that $d^{\lambda,q,\beta}_M$ is non negative ($d^{\lambda,q,\beta}_M$ is lowerbounded $\forall \beta \geq 0$).

\section{Efficient \trot~optimizers}\label{sec-soto}

The key idea behind Sinkhorn-Cuturi's solution is that the KKT conditions ensure that the optimal transportation plan $P^\star$ satisfies $P^\star = \text{diag}(\ve{u})\exp(-\lambda M)\text{diag}(\ve{v})$. Sinkhorn's balancing normalization can then directly be used for a fast approximation of $P^\star$ \cite{sDE,sAR}. This trick does not fit at first sight for Tsallis regularization because the $q$-exponential is \textit{not} multiplicative for general $q$ and KKT conditions do not seem to be as favorable. We give however workarounds for the optimization, that work for \textit{any} $q\in \mathbb{R}_+$.

First, we assume wlog that $q\neq 0, 1$ since in those cases, any efficient LP solver ($q=0$) or Sinkhorn balancing normalization ($q=1$) can be used. The task is non trivial because for $q\in (0,1)$, the function minimized in $d^{\lambda,q}_M$ is \textit{not Lipschitz}, which impedes the convergence of gradient methods. In this case, our workaround is Algorithm \ref{alg:SecondOrderSinkhorn} (\sotrot), which relies on a Second Order approximation of a fundamental quantity used in its convergence proof, auxiliary functions \cite{ddlIF}.
\begin{algorithm}[!ht]
\caption{Second Order Row--\trot~(\sotrot)}\label{alg:SecondOrderSinkhorn}
\textbf{Input:}  marginal $\ve{r}$, matrix $M$, params $\lambda \in \mathbb{R}_{+*}, q\in (0,1)$

\begin{algorithmic}[1]
\STATE $A \gets \lambda M$
\STATE {$P \gets \exp_q(-1)\exp_q^{-1}(A)$}
\STATE \textbf{repeat} 
\STATE \quad $P_1 \gets P\oslash A, P_2 \gets P_1 \oslash A$ \quad //$\oslash$ = Kronecker divide
\STATE \quad $\ve{d} \gets \ve{r} - P\ve{1}, \ve{b} \gets P_1\ve{1}, \ve{a} \gets (2-q)P_2\ve{1}$
\STATE \quad \textbf{for} {$i = 1, 2, ..., n$}
\STATE \quad \quad \textbf{if} {$d_i \geq 0$} \textbf{then}
\STATE \quad \quad \quad {$y_i \gets   \frac{- b_i + \sqrt{b_i^2 + 4a_id_i}}{2a_i}$}
\STATE \quad \quad \textbf{else} 
\STATE \quad \quad \quad {$y_i \gets  d_i/b_i$}
\STATE \quad \quad \textbf{end if}
\STATE \quad \quad \textbf{if} $|y_i| > \frac{q}{(6-4q)\cdot \max_j p_{ij}^{1-q}}$ \textbf{then} 
\begin{eqnarray}
y_i & \gets & \frac{q \cdot \mathrm{sign}(r_i - \sum_j p_{ij})}{(6-4q)\cdot \max_j p_{ij}^{1-q}}\:\:.
\end{eqnarray}
\STATE \quad {$A \gets A -\ve{y}\ve{1}^\top$}
\STATE \quad {$P \gets \exp_q(-1)\exp_q^{-1}(A)$}
\STATE \textbf{until} convergence
\end{algorithmic}
\textbf{Output:} $P$  
\end{algorithm}
\begin{theorem}[Convergence of \sotrot]
\label{thSortrot}
For any fixed $q\in (0,1)$, matrix $P$ output by \sotrot~converges to $P^\star$ with:
\begin{eqnarray}
  P^\star & = & \arg\min_{P \in \mathbb{R}_+^{n\times n} : P\ve{1} = \ve{r}} K_{1/q}(P^q , \tilde{U}^q) \:\:.\nonumber
\end{eqnarray}
\end{theorem}
The proof (in Supplementary Material, Section \ref{proof_thSortrot}) is involved but interesting in itself because it represents one of the first use of the theory of auxiliary functions outside the realm of Bregman divergences in machine learning \cite{cssLRj,ddlIF}. Some important remarks should be made. First, since \sotrot~uses only one of the two marginal constraints, it would need to be iterated ("\textit{wrapped}"), swapping the row and column constraints like in Sinkhorn balancing. In practice, this is not efficient. Furthermore, iterating \sotrot~over constraint swapping does not necessarily converge. For these reasons, we swap constraints \textit{in} the algorithm, making one iteration of Steps 4-14 over rows, and then one iteration of Steps 4-14 over columns (this boils down to transposing matrices in \sotrot), and so on. This converges, but still is not the most efficient. To improve efficiency we perform two modifications, that do not impede convergence experimentally. First, we remove Step 12. In doing so, we not only save $O(n^2)$ computations for \textit{each} outer loop, we essentially make \sotrot~as parallelizable as Sinkhorn balancing \cite{cSD}. Second, we remarked experimentally that convergence is faster when multiplying $y_i$ by 2 in Step 10, and dividing $a$ by 2 in Step 5.

For simplicity, we still refer to this algorithm (balancing constraints in the algorithm, with the modifications for Steps 5, 10, 12) as \sotrot~in the experiments. 

Last, when $q\geq 1$, the function minimized in $d^{\lambda,q}_M$ becomes Lipschitz. In this case, we take the particular geometry of $U(r,c)$ into account by using mirror gradient methods, which are equivalent to gradient methods projected according to some suitable divergence \cite{btMD}. In our case, we consider Kullback-Leibler divergence, which can save a factor $O(n/\sqrt{\log n})$ iterations \cite{btMD}. Furthermore, the Kullback-Leibler projection can be written in terms of Sinkhorn-Knopp's (SK) algorithm with marginals constraints $\ve{r}, \ve{c}$ \cite{skCN}, as is shown in Algorithm \ref{alg:KL}, named \kltrot~($\otimes$ is Kronecker product).

\begin{algorithm}[t]
\caption{KL Projected Gradient --\trot~(\kltrot)}\label{alg:KL}
\textbf{Input:} {Marginals $\ve{r}, \ve{c}$, Matrix $\tilde{U}$, Gradient steps $\lbrace t_k \rbrace$}

\begin{algorithmic}[1]
\STATE $P^{(0)} \gets \tilde{U}$
\STATE \textbf{repeat} 
\STATE \quad $P^{(k+1)}\gets \mbox{{\small SK}}(P^{(k)}\otimes \exp(-t_k\nabla f_q(P^{(k)})), \ve{r}, \ve{c})$
\STATE \textbf{until} {convergence}
\end{algorithmic}
\end{algorithm}

\begin{theorem}
If $q>1$ and the gradient steps $\lbrace t_k \rbrace$ are s.t.  $\sum_k t_k \rightarrow \infty$ and $\sum_k t_k^2 \ll \infty$, matrix $P$ output by \kltrot~converges to $P^\star$ with:
\begin{eqnarray}
  P^\star & = & \arg\min_{P \in U(\ve{r}, \ve{c})} K_{1/q}(P^q , \tilde{U}^q) \:\:.\nonumber
\end{eqnarray}
\end{theorem}
(proof omitted, follows \cite{btMD,skCN})
\section{Experiments}\label{sec-exp}

We evaluate empirically the \trot~framework with its application to ecological inference. The dataset we use describes about $10$ millions individual voters from Florida for the 2012 US presidential elections, as obtained from \cite{ikIE}. The data is much richer than is required for ecological inference: surely we could estimate the joint distribution of every voters' available attributes by counting. This is itself a particularly rare case of data quality in political science, where any analysis is often carried out on aggregate measurements. In fact, since ground truth distributions are effectively available, the Florida dataset has been used to test methodological advances in the field \cite{fwsWS,ikIE}. As a demonstrative example, we focus on inferring the distributions of ethnicity and party for all Florida counties.

\textit{Dataset description and preprocessing.}  The data contains the following attributes \emph{for each voter}: location (district, county), gender, age, party (Democrat, Republican, Other), ethnicity (White, African-american, Hispanic, Asian, Native, Other), 2008 vote (yes, no). About 800K voters with missing attributes are excluded from the study. 
Thanks to the richness of the data, marginal probabilities of ethnic groups and parties can be obtained by counting: for each county we obtain marginals $\bm{r}, \bm{c}$ for the optimal transport problems.

\newcommand{\msur}{M^{\mbox{\tiny{sur}}}}
\newcommand{\mnop}{M^{\mbox{\tiny{no}}}}
\newcommand{\mrbf}{M^{\mbox{\tiny{\textsc{rbf}}}}}

\textit{Evaluation assumptions.} Two assumptions are made in terms of information available for inference. First, the ground truth joint distributions for one district are known; we chose district number $3$ which groups $9$ out of $68$ counties of about $285K$ voters in total. This information will be used to tune hyper-parameters. Second, a cost matrix $\mrbf$ is computed based on mean voter's attributes at state level. For the sake of simplicity, we retain only age (normalized in $[0, 1]$), gender and the 2008 vote; notice that in practice geographical attributes may encode relevant information for computing distances between voter behaviours \cite{fwsWS}. We do not use this. For distance matrix $\mrbf$, we aggregate those features over all Florida for each party to obtain the vectors $\bm{\mu}^{\lab{p}}$ of the party's expected profile and for each ethnic group to obtain the vectors $\bm{\mu}^{\lab{e}}$ of the ethnicity's expected profile. The dissimilarity measure relies on a Gaussian kernel between average county profiles:
\begin{eqnarray}
m^{\mbox{\tiny{\textsc{rbf}}}}_{ij} &\defeq & \sqrt{2- 2 \exp( - \gamma \cdot \| \bm{\mu}^{\lab{p}}_i - \bm{\mu}^{\lab{e}}_j \|_2 )}\:\:,
\end{eqnarray}
\noindent with $\gamma = 10$. The given function is actually the Hilbert metric in the RBF space. Table \ref{table:cost-matrix} shows the resulting cost matrix. Notice how it does encode some common-sense knowledge: White and Republican is the best match, while Hispanic and Asians are the worst match with Republican profiles. It is rather surprising that only 3 features such as age, gender and whether people voted at the last election can reflect so well those relative political traits; these results are indeed much in line with survey-based statistics \cite{gallup}. We also try another cost matrix $M$, $\msur$, derived from the ID proportions of parties composition given in \cite{gallup}; $m^{\mbox{\tiny{sur}}}_{ij}$ is computed as $1 - p_{ij}$, where $p_{ij}$ is the proportion of people registered to party $j$ belonging to ethnic group $i$. Finally, we consider a "no prior" matrix $\mnop$, in which $m^{\mbox{\tiny{no}}}_{ij} = 1, \forall i,j$.
 
\textit{Cross-validation of $q$}. We study the solution of \trot~for a grid of $\lambda \in [0.01, 1000], q \in [0.5, 4]$, inferring the joint distributions of all counties of district number 3. We measure average KL-divergence between inferred and ground truth joint distributions. Notice that each county defines a \textit{different} optimal transport problem; inferring the joint distributions for multiple counties at a time is therefore trivial to parallelize. This is somewhat counter-intuitive since we may believe that geographically wider spread data should improve inference at a local level, that is, more data better inference. Indeed, the implicit coupling of the problem is represented by cost matrix, which expresses some prior knowledge of the problem by means of all data from Florida. 

\begin{table}[t]
	\scriptsize
	\centering
	\begin{tabular}{l|cccccc}
	\toprule
	  \diagbox{party}{ethnicity} & white & afro. & hispanic & asian & native & other \\
	\hline
        Democrat & $0.29$ & $0.38$ & $0.55$ & $0.55$ &  $0.37$ & $0.57$ \\
        Republican & $0.18$ & $0.63$ & $0.76$ & $0.84$ & $0.54$ & $0.72$ \\
        Other & $0.74$ & $0.62$ &  $0.27$ &  $0.24$ &  $0.41$ & $0.23$ \\
	\bottomrule
	\end{tabular}
	\caption{Visualization of the cost matrix as $M$: small values indicate high similarity. Highest similarity: (white, Republican); lowest similarity: (asian, Republican) followed by (hispanic, Republican).}
	\label{table:cost-matrix}
\end{table}

\begin{table}[t]
	\centering
    \scriptsize
	\begin{tabular}{l|ccccc}
    \toprule
	  Algorithm \hspace{-0.21cm} & \hspace{-0.21cm} $M$ \hspace{-0.21cm} & \hspace{-0.21cm} $q$ \hspace{-0.21cm} & \hspace{-0.21cm} $\lambda$ \hspace{-0.21cm} & \hspace{-0.21cm}  KL-divergence $\pm$ SD \hspace{-0.21cm} & \hspace{-0.21cm} Abs. error $\pm$ SD \\
	\midrule
        {\tiny Florida-Average} \hspace{-0.21cm} & \hspace{-0.21cm} - \hspace{-0.21cm} & \hspace{-0.21cm} - \hspace{-0.21cm} & \hspace{-0.21cm} - \hspace{-0.21cm} & \hspace{-0.21cm} $0.251 \pm 0.187$ \hspace{-0.21cm} & \hspace{-0.21cm} $0.025 \pm 0.011$\\
        \hline
        Simplex \hspace{-0.21cm} & \hspace{-0.21cm} $\mrbf$ \hspace{-0.21cm} & \hspace{-0.21cm} - \hspace{-0.21cm} & \hspace{-0.21cm} - \hspace{-0.21cm} & \hspace{-0.21cm} $0.280 \pm 0.108$ \hspace{-0.21cm} & \hspace{-0.21cm} $0.023 \pm 0.008$\\
        Simplex \hspace{-0.21cm} & \hspace{-0.21cm} $\msur$ \hspace{-0.21cm} & \hspace{-0.21cm} - \hspace{-0.21cm} & \hspace{-0.21cm} - \hspace{-0.21cm} & \hspace{-0.21cm} $0.136 \pm 0.098$ \hspace{-0.21cm} & \hspace{-0.21cm} $0.013  \pm 0.009$\\
        \hline
        Sinkhorn \hspace{-0.21cm} & \hspace{-0.21cm} $\mrbf$ \hspace{-0.21cm} & \hspace{-0.21cm} $\:\: 1.0^{\dagger}$ \hspace{-0.21cm} & \hspace{-0.21cm} $10^{0}$ \hspace{-0.21cm} & \hspace{-0.21cm} $0.054 \pm 0.036$ \hspace{-0.21cm} & \hspace{-0.21cm} $0.009 \pm 0.005$\\
        Sinkhorn \hspace{-0.21cm} & \hspace{-0.21cm} $\msur$ \hspace{-0.21cm} & \hspace{-0.21cm} $\:\: 1.0^{\dagger}$ \hspace{-0.21cm} & \hspace{-0.21cm} $10^{1}$ \hspace{-0.21cm} & \hspace{-0.21cm} $0.035 \pm 0.027$ \hspace{-0.21cm} & \hspace{-0.21cm} $0.007 \pm 0.004$\\
\hline
        \trot \hspace{-0.21cm} & \hspace{-0.21cm} $\mrbf$ \hspace{-0.21cm} & \hspace{-0.21cm} $1.0$ \hspace{-0.21cm} & \hspace{-0.21cm} $10^{0}$ \hspace{-0.21cm} & \hspace{-0.21cm} $0.054 \pm 0.036$ \hspace{-0.21cm} & \hspace{-0.21cm} $0.009 \pm 0.005$\\
        \trot \hspace{-0.21cm} & \hspace{-0.21cm} $\msur$ \hspace{-0.21cm} & \hspace{-0.21cm} $2.8$ \hspace{-0.21cm} & \hspace{-0.21cm} $10^1$ \hspace{-0.21cm} & \hspace{-0.21cm} $\bm{0.007 \pm 0.009}$ \hspace{-0.21cm} & \hspace{-0.21cm} $\bm{0.003 \pm 0.002}$\\ \hline
        \trot \hspace{-0.21cm} & \hspace{-0.21cm}~$\mnop$ \hspace{-0.21cm} & \hspace{-0.21cm} $0.8$ \hspace{-0.21cm} & \hspace{-0.21cm} $10^{0}$ \hspace{-0.21cm} & \hspace{-0.21cm} $0.076 \pm 0.048$ \hspace{-0.21cm} & \hspace{-0.21cm} $ 0.011 \pm 0.005$ \\
	\bottomrule
	\end{tabular}
	\caption{Average KL-divergence and absolute error with standard deviation (SD) of algorithms inferring joint distributions of all Florida counties. Parameters noted with $\dagger$ are not cross-validated but defined by the algorithm.}
	\label{table:kl}
\end{table}

\textit{Baselines and comparisons with other methods.} To evaluate quantitatively the solution of \trot~is useful to define a set of baseline methods: i) Florida-average, which the same state-level joint distribution (assumed prior knowledge) for each of the 67 county; ii) Simplex, that is the solution of optimal transport with no regularization as given by the Simplex algorithm; iii) Sinkhorn(-Cuturi)'s algorithm, which is \trot~with $q=1$; iv) \trot. ii-iv are tested with $M \in \{\mrbf, \msur\}$, and we provide in addition the results for \trot~with $M=\mnop$. Hyper-parameters are cross-validated independently for each algorithm. 

Table \ref{table:kl} reports a quantitative comparison. From the most general to the most specific, there are three remarks to make. First, optimal transport can be (but is not always) better than the default distribution (Florida average). Second, 
\textit{regularizing} optimal transport consistently improves upon these baselines. Third, \trot~successfully matches Sinkhorn's approach when $q=1$ is be the best solution in \trot's range of $q$ ($M = \mrbf$), and 
manages to tune $q$ to significantly beat Sinkhorn's when better alternatives exist: with $M=\msur$, \trot~divides the expected KL divergence by more than \textit{seven} (7) compared to Sinkhorn. This is a strong advocacy to allow for the tuning of $q$. Notice that in this case, $\lambda$ is larger compared to $M=\mrbf$, which makes sense since $M=\msur$ is more accurate for the optimal transport problem (see the Simplex results) and so the weight of the regularizer predictably decreases in the regularized optimal transport distance. We conjecture that $M=\msur$ beats $M=\mrbf$ in part because it is somehow finer grained: $\mrbf$ is computed from sufficient statistics for the marginals alone, while $\msur$ exploits information computed from the cartesian product of the supports.
Figure \ref{fig:scatter} compares all 1 836 inferred probabilities ($3 \times 6$ per county) with respect to the ground truth for Sinkhorn vs \trot~using $M=\msur$. Remark that the figures in Table \ref{table:kl} translate to per-county ecological inference results that are significantly more in favor of \trot, which basically has no "hard-to-guess" counties compared to Sinkhorn for which the absolute difference between inference and ground truth can exceed 10$\%$.

To finish up, additional experiments, displayed in \sm~(Sections \ref{exp_expFAV} and \ref{exp_expSURV}) also show that \trot~with $M=\msur$ manages to have a distribution of per county errors extremely peaked around zero error, compared to the simplest baselines (Florida average and \trot~with $M=\mnop$). These are good news, but there are some local discrepancies. For example, there exists \textit{one} county on which \trot~with $M=\msur$ is beaten by \trot~with $M=\mnop$.

\begin{figure}[t]
  \centering
    \includegraphics[trim=30bp 10bp 40bp 40bp,clip,width=0.9\columnwidth]{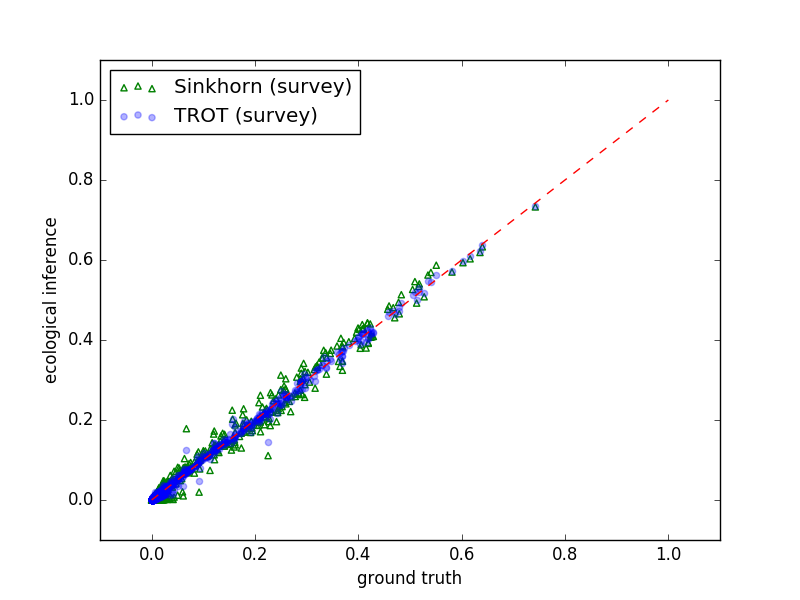}
    \caption{Correlation between \trot~vs Sinkhorn inferred probabilities and ground truth for all Florida counties (the closer to $y=x$, the better).}
    \label{fig:scatter}
\end{figure}

\section{Discussion and conclusion}\label{sec-con}

In this paper, we have bridged Shannon regularized optimal transport and unregularized optimal transport, via Tsallis entropic regularization.
There are three main motivations to the generalization, the two first have already been discussed: \trot~allows to keep the properties of Sinkhorn distances, and fields like ecological inference bring natural applications for the general \trot~family. The application to ecological inference is also interesting because the main unknown is the optimal transportation plan and not necessarily the transportation distance obtained.
The third and last motivation is important for applications at large and ecological inference in particular. \trot~spans a subset of $f$-divergences, and $f$-divergences satisfy the information monotonicity property that coarse graining does not increase the divergence \cite[$\S$ 3.2]{aIG}. Furthermore, $f$-divergences are invariant under diffeomorphic transformations \cite[Theorem 1]{qmAS}. This is a powerful statement: if the ground metric is affected by such a transformation $h$ (for example, we change the underlying manifold coordinate system, \textit{e.g.} for privacy reasons), then, from the optimal \trot~transportation plan $P^\star$, the transportation plan corresponding to the initial coordinate system can be recovered from the \textit{sole} knowledge of $h^{-1}$. 

The algorithms we provide allow for the efficient optimization of the regularized optimal transport for all values of $q\geq 0$, and include notable cases for which conventional gradient-based approaches would probably not be the best approaches due to the fact that the function to optimize is not Lipschitz for the $q$ chosen. In fact, the main notable downside of the generalization is that we could not prove the same (geometric) convergence rates as the ones that are known for Sinkhorn's approach \cite{flOT}. 

Our results display that there can be significant discrepancies in the regularized optimal transport results depending on how cost matrix $M$ is crafted, yet the information we used for our best experiments is readily available from public statistics (matrices $\mrbf, \msur$). Even the instantiation without prior knowledge ($M = \ve{1}\ve{1}^\top$) does not strictly fail in returning useful solutions (compared \textit{e.g.} to Florida average and unregularized optimal transport). This may be a strong advocacy to use \trot~even on domains for which little prior knowledge is available.

\section*{Acknowledgments}

The authors wish to thank Seth Flaxman and Wendy K. Tam Cho for numerous stimulating discussions. Work done while Boris Muzellec was visiting Nicta / Data61. Nicta was funded by the Australian Government through the Department of Communications and the Australian Research Council through the ICT Center of Excellence Program.

\bibliography{bibgen}

\newpage 

\section{Supplementary Material --- Table of contents}

\noindent \textbf{Supplementary material on proofs} \hrulefill Pg \pageref{proof_proofs}\\
\noindent Proof of Theorem \ref{thDist}\hrulefill Pg \pageref{proof_thDist}\\
\noindent Proof of Theorem \ref{thUnique}\hrulefill Pg \pageref{proof_thUnique}\\
\noindent Proof of Theorems \ref{thAxiom} and \ref{thdist3}\hrulefill Pg \pageref{proof_thAxiom}\\
\noindent Proof of Theorem \ref{thSortrot}\hrulefill Pg \pageref{proof_thSortrot}\\

\noindent \textbf{Supplementary material on experiments} \hrulefill Pg
\pageref{exp_expes}\\
\noindent Per county error distribution, \trot~survey vs Florida average\hrulefill Pg \pageref{exp_expFAV}\\
\noindent Per county errors, \trot~survey vs \trot~$\ve{1}\ve{1}^\top$\hrulefill Pg \pageref{exp_expSURV}

\newpage

\section*{Supplementary Material: proofs}\label{proof_proofs}

\section{Proof of Theorem \ref{thDist}}\label{proof_thDist}

Let $M\in\mathbb{R}_+^{n\times n}$ be a distance matrix, and $q,q'\in\mathbb{R} - \lbrace 1\rbrace, q \neq q'$ (the case when $q=1$ xor $q'=1$ can be treated in a similar fashion). We suppose wlog that the support does not reduce to a singleton (otherwise the solution to optimal transport is trivial).
Rescaling $M$ and a constant row vector and a constant column vector, the solution of \trot$(q,\lambda,M)$ can be written wlog as
\begin{eqnarray}
p_{ij} & = & \exp_q(-1)\exp_q^{-1}(m_{ij})\:\:.
\end{eqnarray}
Assume there exists a $\lambda'\in\mathbb{R}$ such that the solution of \trot$(q',\lambda',M)$ is equal to that of \trot$(q,\lambda,M)$. This is equivalent to saying that there exists $\bm{\alpha},\bm{\beta}\in\mathbb{R}^n$ such that
\begin{eqnarray}
\exp_q(m_{ij}) & = & \exp_{q'}(\alpha_i+\lambda'm_{ij}+\beta_j)\:\:, \forall i,j\:\:.
\end{eqnarray}
Composing with $\log_{q'}$ and rearranging, this implies that
\begin{eqnarray}
f_{q',q}^{\lambda'}(m_{ij}) & = & \alpha_i+\beta_j\:\:,\forall i,j\:\:,\label{defFF}
\end{eqnarray}
where
\begin{eqnarray}
f_{q',q}^{\lambda'}(x) & \defeq & \log_{q'}\circ\exp_q - \lambda'\mathrm{Id}\:\:.
\end{eqnarray}
Now, remark that, since $M$ is a distance, $m_{ii} = 0, \forall i$ because of the identity of the indiscernibles, and so $\alpha_i+\beta_i = f_{q',q}^{\lambda'}(0) = 0$, implying $\ve{\alpha}=-\ve{\beta}$. $f_{q',q}^{\lambda'}$ is differentiable. Let:
\begin{eqnarray}
g_{q',q}^{\lambda'} (x) & \defeq & \frac{\mathrm{d}}{\mathrm{d}x} f_{q',q}^{\lambda'} (x) \nonumber\\
 & = & \exp_q^{q-q'}(x) - \lambda'\:\:;\\
h_{q',q}^{\lambda'} (x) & \defeq & \frac{\mathrm{d}}{\mathrm{d}x} g_{q',q}^{\lambda'} (x) \nonumber\\
 & = & (q-q')\cdot \exp_q^{2q-q'-1}(x)\:\:.
\end{eqnarray}
If we assume wlog that $q>q'$, then $g_{q',q}^{\lambda'}$ is increasing and zeroes at most once over $\mathbb{R}$, eventually on some $m^*$ that we define as $m^* = \log_q\left({\lambda'}^{\frac{1}{q-q'}}\right)$ if $(\lambda'>1) \wedge (0\in \mathrm{Im} g_{q',q}^{\lambda'})$ (and $+\infty$ otherwise).
Notice that $m^* >0$ and $f_{q',q}^{\lambda'}$ is bijective over $(0,m^*)$. Suppose wlog that $m_{ij}\leq m^*, \forall i, j$. Otherwise, all distances are scaled by the same real so that $m_{ij}\leq m^*, \forall i, j$: this does not alter the property of $M$ being a distance. A distance being symmetric, we also have $m_{ij} = m_{ji}$ and since $f_{q',q}^{\lambda'}$ is strictly increasing in the range of distances, then we get from eq. (\ref{defFF}) that $\alpha_i+\beta_j = \alpha_j+\beta_i, \forall i, j$ and so $\alpha_i - \alpha_j = \beta_i - \beta_j = -(\alpha_i - \alpha_j)$ (since $\ve{\alpha}=-\ve{\beta}$). Hence, there exists a real $\alpha$ such that $\ve{\alpha} = \alpha\cdot \ve{1}$. We get, in matrix form
\begin{eqnarray}
f_{q',q}^{\lambda'}(M) & = & \ve{\alpha}\ve{1}^\top + \ve{1}\ve{\beta}^\top\label{eqRank}\\
 & = & \alpha\cdot \ve{1}\ve{1}^\top - \alpha\cdot \ve{1} \ve{1}^\top = 0\:\:.
\end{eqnarray}
Hence, $m_{ij} = m_{ii}, \forall i,j$ and the support reduces to a singleton (because of the identity of the indiscernibles), which is impossible.

Remark that the proof also works when $M$ is not a distance anymore, but for example contains all arbitrary non negative matrices. To see this, we remark that the right hand side of eq. (\ref{eqRank}) is a matrix of rank no larger than 2. Since $f_{q',q}^{\lambda'}$ is continuous, we have 
\begin{eqnarray*}
\mathrm{Im}(f_{q',q}^{\lambda'}) & \defeq & \mathcal{I} \subseteq \mathbb{R}
\end{eqnarray*} 
where $\mathcal{I}$ is not reduced to a singleton and so the left hand side of eq. (\ref{eqRank}) spans matrices of arbitrary rank. Hence, eq. (\ref{eqRank}) cannot always hold. 

\section{Proof of Theorem \ref{thUnique}}\label{proof_thUnique}
Denote
\begin{eqnarray*}
f_{ij}: p_{ij}\rightarrow p_{ij}m_{ij} -\frac{1}{\lambda (1-q)}(p_{ij}^q-p_{ij})\:\:.
\end{eqnarray*}
$f_{ij}$ is twice differentiable on $\mathbb{R}_{+*}$, and 
\begin{eqnarray*}
\frac{\mathrm{d}^2}{\mathrm{d}x^2}f_{ij}(x) & = & \frac{q}{\lambda}x^{q-2} > 0
\end{eqnarray*}
for any fixed $q>0$, and so $f_{ij}$ is strictly convex on $\mathbb{R}_{+*}$. We also remark that $U(\ve{r},\ve{c})$ is a non-empty compact subset of $\mathbb{R}^{n \times n}$. Indeed, $\ve{r}\ve{c}^\top \in U(\ve{r},\ve{c})$, $\forall P\in U(\ve{r},\ve{c}), \Vert P\Vert_1 = 1$ (which proves boundedness) and $U(\ve{r},\ve{c})$ is a closed subset of $U(\ve{r},\ve{c})$ (being the intersection of the pre-images of singletons by continuous functions). Hence, since $\inner{P}{M} -\frac{1}{\lambda}H_q(P) = \sum_{i,j}f_{ij}(p_{ij})$, there exists a unique minimum of this function in $U(\ve{r},\ve{c})$.\\

To prove the analytic shape of the solution, we remark that \trot($q,\lambda,M$) consists in minimizing a convex function given a set of affine constraints, and so the KKT conditions are necessary and sufficient. The KKT conditions give
\begin{eqnarray*}
p_{ij} & = & \exp_q(-1)\exp_q^{-1}(\alpha_i + \lambda m_{ij} + \beta_j)\:\:,
\end{eqnarray*}
where $\ve{\alpha}, \ve{\beta} \in \mathbb{R}^n$ are Lagrange multipliers. 

Finally, let us show that Lagrange multipliers $\ve{\alpha}, \ve{\beta}\in \mathbb{R}^n$ are unique up to an additive constant. Assume that $\ve{\alpha},\ve{\alpha}',\ve{\beta},\ve{\beta}' \in \mathbb{R}^n$ are such that 
	
\begin{align*}
\forall i,j, p_{ij} &= \exp_q(-1)\exp_q^{-1}(\lambda m_{ij} + \alpha_i + \beta_j)\\
&= \exp_q(-1)\exp_q^{-1}(\lambda m_{ij} + \alpha_i' + \beta_j')\:\:,
\end{align*}
where $P$ is the unique solution of \trot($q,\lambda,M$). This implies 
\begin{eqnarray*}
\alpha_i + \beta_j & = & \alpha_i' + \beta_j'\:\:, \forall i,j\:\:,
\end{eqnarray*}
\textit{i.e.} 
\begin{eqnarray*}
\alpha_i -\alpha_i' & = & \beta_j' - \beta_j\:\:, \forall i,j\:\:.
\end{eqnarray*} 
In particular, if there exists $i_0$ and $C\neq 0$ such that $\alpha_{i_0} -\alpha_{i_0}' = C$, then $\forall j, \beta_j' = \beta_j + C$ and in turn $\forall i, \alpha_i = \alpha_i'+ C$, which proves our claim.

\section{Proof of Theorems \ref{thAxiom} and \ref{thdist3}}\label{proof_thAxiom}

For reasons that we explain now, we will in fact prove Theorem \ref{thdist3} before we prove Theorem \ref{thAxiom}.

\noindent Had we chosen to follow \cite{cSD}, we would have replaced \trot($q,\lambda,M$) by:
\begin{eqnarray}
d_{M,\alpha, q}(\ve{r},\ve{c}) & \defeq & \min_{\substack{P \in U(\ve{r},\ve{c})\\ H_q(P) - H_q(\ve{r}) - H_q(\ve{c}) \geq \alpha}} \inner{P}{M}\:\:,\label{trot2}
\end{eqnarray}
for some $\alpha > 0$. Both problems are equivalent since $\lambda$ in \trot($q,\lambda,M$) plays the role of the Lagrange multiplier for the entropy constraint in eq. (\ref{trot2}) \cite[Section 3]{cSD}, and so \textit{there exists} an equivalent value of $\alpha^*$ for which both problems coincide:
\begin{eqnarray}
d_{M,\alpha^*, q}(\ve{r},\ve{c}) & = & d^{\lambda,q}_M (\ve{r},\ve{c})\:\:,\label{const1}
\end{eqnarray}
so eq. (\ref{trot2}) indeed matches \trot($q,\lambda,M$). It is clear from eq. (\ref{const1}) that $\alpha$ does not depend solely on $\lambda$, \textit{but also} (eventually) on all other parameters, including $\ve{r},\ve{c}$.

This would not be a problem to state the triangle inequality \textit{for} $d_{M,\alpha, q}$, as in \cite{cSD} ($\forall \ve{x}, \ve{y}, \ve{z} \in \bigtriangleup_n$):
\begin{eqnarray}
d_{M, \alpha, q}(\ve{x},\ve{z}) & \leq & d_{M, \alpha, q}(\ve{x},\ve{y})+d_{M, \alpha, q}(\ve{y},\ve{z})\:\:.\label{defdist}
\end{eqnarray}
However, $\alpha$ is \textit{fixed} and in particular different from the $\alpha^*$ that guarantee eq. (\ref{const1}) --- and there might be three different sets of parameters for $d^{\lambda,q}_M$ as it would equivalently appear from eq. (\ref{defdist}). Under the simplifying assumption that only $\lambda$ changes, we might just get from eq. (\ref{defdist}):
\begin{eqnarray}
d^{\lambda^*,q}_M (\ve{x},\ve{z}) & \leq & d^{{\lambda'}^*,q}_M (\ve{x},\ve{y}) + d^{{\lambda''}^*,q}_M (\ve{y},\ve{z})\:\:, 
\end{eqnarray}
with ${\lambda}^*\neq {\lambda'}^*\neq {\lambda''}^*$. Worse, the transportation plans may change with $\lambda$: for example, we may have
\begin{eqnarray*}
\arg\min_{P \in  U(\ve{x},\ve{z})} d^{\lambda_1,q}_M (\ve{x},\ve{z}) & \neq & \arg\min_{P \in  U(\ve{x},\ve{z})} d^{\lambda_2,q}_M (\ve{x},\ve{z}) \:\:,
\end{eqnarray*}
with $\lambda_1 \neq \lambda_2$ and $\lambda_1, \lambda_2 \in \{{\lambda}^*, {\lambda'}^*, {\lambda''}^*\}$. So, the triangle inequality for $d^{\lambda,q}_M$ that follows from ineq. (\ref{defdist}) does not allow to control the parameters of \trot($q,\lambda,M$) nor the optimal transportation plans that follows. It does not show a problem in regularizing the optimal transport distance, but rather that the distance $d_{M,\alpha, q}$ chosen from eq. (\ref{const1}) does not completely fulfill its objective in showing that regularization in $d^{\lambda,q}_M$ still keeps some of the attractive properties that unregularized optimal transport meets.

To bypass this problem and establish a statement involving a distance in which all parameters are in the clear and optimal transportation plans still coincide with $d^{\lambda,q}_M$, we chose to rely on measure:
\begin{eqnarray}
d^{\lambda,q,\beta}_M (\ve{r},\ve{c}) & \defeq & \min_{P \in  U(\ve{r},\ve{c})} \inner{P}{M} \nonumber\\
 & & - \frac{1}{\lambda}\cdot \left(H_q(P) - \beta\cdot( H_q(\ve{r}) + H_q(\ve{c}))\right)\:\:,\nonumber
\end{eqnarray}
where $\beta$ is some \textit{constant}. There is one trivial but crucial fact about $d^{\lambda,q,\beta}_M (\ve{r},\ve{c})$: regardless of the choice of $\beta$, its optimal transportation plan is the \textit{same} as for \trot($q,\lambda,M$).
\begin{lemma}
For any $\ve{r}, \ve{c} \in \bigtriangleup_n$ and constant $\beta \in \mathbb{R}$, let
\begin{eqnarray}
P_1 & \defeq & \arg\min_{P \in  U(\ve{r},\ve{c})} \inner{P}{M} \nonumber\\
 & & - \frac{1}{\lambda}\cdot \left(H_q(P) - \beta\cdot( H_q(\ve{r}) + H_q(\ve{c}))\right)\:\:.\\
P_2 & \defeq & \arg\min_{P \in  U(\ve{r},\ve{c})} \inner{P}{M} \nonumber\\
 & & - \frac{1}{\lambda}\cdot \left(H_q(P) \right)\:\:.
\end{eqnarray}
Then $P_1 = P_2$.
\end{lemma}
\begin{theorem}\label{thdist2}
The following holds for any fixed $q\geq 1$ (unless otherwise stated):
\begin{itemize}
\item for \textbf{any} $\beta \geq 1$, $d^{\lambda,1,\beta}_M$ satisfies the triangle inequality;
\item for the choice $\beta = 1/2$, $d^{\lambda,q,1/2}_M$ satisfies the following weak version of the identity of the indiscernibles: if $\ve{r} = \ve{c}$, then $d^{\lambda,q,1/2}_M(\ve{r},\ve{c})\leq 0$.
\item for the choice $\beta = 1/2$, $\forall \ve{r} \in \bigtriangleup_n$, choosing the (no) transportation plan $P = \mathrm{Diag}(\ve{r})$ brings
\begin{eqnarray*}
\inner{P}{M} - \frac{1}{\lambda}\cdot \left(H_q(P) - \frac{1}{2}\cdot( H_q(\ve{r}) + H_q(\ve{r}))\right) & = & 0\:\:.
\end{eqnarray*}
\end{itemize}
\end{theorem}
\noindent \textbf{Remark}: the last property is trivial but worth stating since the (no) transportation plan $P = \mathrm{Diag}(\ve{r})$ also satisfies $P = \arg \min_{Q \in U(\ve{r},\ve{r})} \inner{Q}{M}$, which zeroes the (no) transportation distance $d_M(\ve{r}, \ve{r})$.

\begin{proof}
To prove the Theorem, we need another version of the Gluing Lemma with entropic constraints \cite[Lemma 1]{cSD}, generalized to handle Tsallis entropy.
\begin{lemma}(Refined gluing Lemma)\label{lemglu}
Let $\ve{x}, \ve{y}, \ve{z}\in \bigtriangleup_n$. Let $P \in U(\ve{x}, \ve{y})$ and $Q\in U(\ve{y}, \ve{z})$. Let $S\in \mathbb{R}^{n\times n}$ defined by general term 
\begin{eqnarray}
s_{ik} & \defeq & \sum_j\frac{p_{ij}q_{jk}}{y_j}\:\:.
\end{eqnarray}
The following holds about $S$:
\begin{enumerate}
\item $S\in U(\ve{x},\ve{z})$;
\item if $q\geq 1$, then:
\begin{eqnarray}
\lefteqn{H_q(S) - H_q(\ve{x}) - H_q(\ve{z})}\nonumber\\
 & \geq & H_q(P) - H_q(\ve{x}) - H_q(\ve{y})\label{bp1}\:\:.
\end{eqnarray}
\end{enumerate}
\end{lemma}
\begin{proof}
The proof essentially builds upon \cite[Lemma 1]{cSD}. We remark that $S$ can be built by
\begin{eqnarray}
s_{ik} & = & \sum_j t_{ijk}\:\:,
\end{eqnarray}
where $\forall i,j,k \in \{1, 2, ..., n\}$, we have
\begin{eqnarray}
t_{ijk} & \defeq & 
\frac{p_{ij}q_{jk}}{y_j} 
\end{eqnarray}
if $y_j\neq 0$ (and $t_{ijk}  = 0$ otherwise)

$S$ is a transportation matrix between $\ve{x}$ and $\ve{z}$. Indeed,
\begin{eqnarray*}
\sum_i\sum_j s_{ijk} & = & \sum_j\sum_i \frac{p_{ij}q_{jk}}{y_j} \\
 & = & \sum_j\frac{q_{jk}}{y_j}\sum_i p_{ij} \\
 & =& \sum_j\frac{q_{jk}}{y_j} y_j = \sum_j q_{jk} =z_k\:\:;\\
\sum_k\sum_j s_{ijk} & = & \sum_j\sum_k \frac{p_{ij}q_{jk}}{y_j}\\
 & = & \sum_j\frac{p_{ij}}{y_j}\sum_k q_{jk} \\
 & = & \sum_j\frac{p_{ij}}{y_j} y_j = \sum_j p_{ij} =x_i\:\:.
\end{eqnarray*}
So, $S \in U(\ve{x},\ve{z})$. To prove ineq. (\ref{bp1}), we need the following definition from \cite{fIT}.
\begin{definition}\cite{fIT}\label{defFTE}
Let $\X$ and $\Y$ denote random variables. The Tsallis conditional entropy of $\X$ given $\Y$, and Tsallis joint entropy of $\X$ and $\Y$, are respectively given by:
\begin{eqnarray*}
H_q(\X\vert \Y) & \defeq & -\sum_{x,y} p(x, y)^q \log_qp(x\vert y) \:\:,\\
H_q(\X,\Y) & \defeq & -\sum_{x,y} p(x,y)^q \log_qp(x,y)\:\:.
\end{eqnarray*}
The Tsallis mutual
entropy of $\X$ and $\Y$ is defined by
\begin{eqnarray*}
I_q(\X; \Y) & \defeq & H_q(\X) - H_q(\X|\Y) \\
 & = & H_q(\X) + H_q(\Y) - H_q(\X,\Y)\:\:.
\end{eqnarray*}
\end{definition}
We have made use of the simplifying notation that removes variables names when unambiguous, like $p(x)\defeq p(\X = x)$. Let $\X, \Y, \Z$ be random variables jointly distributed as $T$, that is, for any $x,y,z$,
\begin{eqnarray}
p(x,y,z) & = & \frac{p(x,y)p(y,z)}{p(y)}\label{defJD}
\end{eqnarray}
It follows from that and Bayes rule that:
\begin{eqnarray}
p(x|y) & = & \frac{p(x,y)}{p(y)}\nonumber\\
 & = &  \frac{p(x,y,z)}{p(y,z)} \:\:, \forall z\nonumber\\
 & = & p(x|y,z) \:\:, \forall z\:\:,
\end{eqnarray}
and so
\begin{eqnarray}
I_q(\X; \Z | \Y) & \defeq &  H_q(\X | \Y) - H_q(\X|\Y, \Z)\nonumber\\
 & = & 0\:\:.
\end{eqnarray}
It comes from \cite[Theorem 4.3]{fIT},
\begin{eqnarray}
I_q(\X; \Y, \Z) & = & I_q(\X; \Z) + I_q(\X; \Y | \Z)\\
 & = & I_q(\X; \Y) + I_q(\X; \Z | \Y)\:\:,
\end{eqnarray}
but since $I_q(\X; \Z | \Y) = 0$, we obtain
\begin{eqnarray}
I_q(\X; \Y) & = &  I_q(\X; \Z) + I_q(\X; \Y | \Z)\:\:.
\end{eqnarray}
It also follows from \cite[Theorem 3.4]{fIT} that $I_q(\X; \Y | \Z) \geq 0$ whenever $q\geq 1$, and so
\begin{eqnarray}
I_q(\X; \Y) & \geq &  I_q(\X; \Z) \:\:, \forall q\geq 1\:\:.\label{inHT}
\end{eqnarray}
Now, it comes from Definition \ref{defFTE} and the definition of $\X,\Y$ and $\Z$ from eq. (\ref{defJD}),
\begin{eqnarray}
- I_q(\X; \Y) & = & H_q(\X,\Y) - H_q(\X) - H_q(\Y)\nonumber\\
 & = & H_q(P) - H_q(\ve{x}) - H_q(\ve{y})\:\:,\\
- I_q(\X; \Z) & = & H_q(\X,\Z) - H_q(\X) - H_q(\Z)\nonumber\\
 & = & H_q(S) - H_q(\ve{x}) - H_q(\ve{z})\:\:.
\end{eqnarray}
Since $P\in U_\lambda(\ve{x},\ve{y})$, by assumption, we obtain from ineq. (\ref{inHT}) that whenever $q\geq 1$,
\begin{eqnarray}
H_q(S) - H_q(\ve{x}) - H_q(\ve{z}) & \geq & H_q(P) - H_q(\ve{x}) - H_q(\ve{y})\nonumber\:\:,
\end{eqnarray}
as claimed.
\end{proof}
We can now prove Theorem \ref{thdist2}. Shannon's entropy is denoted $H_1$ for short.

Define for short
\begin{eqnarray}
\Delta & \defeq & H_1(P) + H_1(Q) - H_1(S) - 2\beta\cdot H_1(\ve{y})\:\:,
\end{eqnarray}
where $P, Q, S$ are defined in Lemma \ref{lemglu}.
It follows from the definition of $S$ and \cite[Proof of Theorem 1]{cSD} that
\begin{eqnarray}
 \lefteqn{d^{\lambda,q,\beta}_M (\ve{x},\ve{z})}\nonumber\\
 &  \defeq & \min_{R \in  U(\ve{x},\ve{z})} \inner{R}{M} - \frac{1}{\lambda}\cdot \left(H_1(R) - \beta\cdot( H_1(\ve{x}) + H_1(\ve{z}))\right)\nonumber\\
 & \leq & \inner{S}{M} - \frac{1}{\lambda}\cdot \left( H_1(S)  - \beta\cdot(H_1(\ve{x}) + H_1(\ve{z}))\right)\nonumber\\
 & \leq & \inner{P}{M} + \inner{Q}{M} - \frac{1}{\lambda}\cdot \left(H_1(S)  - \beta\cdot(H_1(\ve{x}) + H_1(\ve{z}))\right) \nonumber\\
 & & = \inner{P}{M} - \frac{1}{\lambda}\cdot \left(H_1(P) - \beta\cdot(H_1(\ve{x}) + H_1(\ve{y}))\right) \nonumber\\
 & & + \inner{Q}{M} - \frac{1}{\lambda}\cdot \left(H_1(Q) - \beta\cdot(H_1(\ve{y}) + H_1(\ve{z}))\right) \nonumber\\
 & & + \frac{1}{\lambda}\cdot (H_1(P) + H_1(Q) - H_1(S) - 2\beta\cdot H_1(\ve{y}))\nonumber\\
 & \defeq & d^{\lambda,q,\beta}_M (\ve{x},\ve{y}) + d^{\lambda,q,\beta}_M (\ve{y},\ve{z}) + \frac{1}{\lambda}\cdot \Delta\:\:.\label{eqdelta1}
\end{eqnarray}
We now show that $\Delta\leq 0$. For this, observe that ineq. (\ref{bp1}) yields:
\begin{eqnarray}
\lefteqn{\Delta}\nonumber\\
 & \leq & (H_1(S) + H_1(\ve{y})-H_1(\ve{z})) \nonumber\\
 & & + H_1(Q) - H_1(S) - 2\beta\cdot H_1(\ve{y}) \nonumber\\
 & & = H_1(Q) - H_1(\ve{y}) - H_1(\ve{z}) + 2 (1-\beta) H_1(\ve{y}) \:\:,\label{bp2}
\end{eqnarray}
and, by definition of $Q, \ve{y}, \ve{z}$,
\begin{eqnarray}
\lefteqn{H_1(Q) - H_1(\ve{y}) - H_1(\ve{z})}\nonumber\\
 & \defeq & H_1(\Y, \Z) - H_1(\Y) - H_1(\Z)\:\:.\label{eq00}
\end{eqnarray}
Shannon's entropy of a joint distribution is maximal with independence: $H_1(\Y, \Z) \leq H_1(\Y\times \Z) = H_1(\Y) + H_1(\Z)$,
so we get from eq. (\ref{bp2}) after simplifying
\begin{eqnarray}
\Delta & \leq & 2 (1-\beta) H_1(\ve{y}) \label{leq1}\:\:.
\end{eqnarray}
Hence if $\beta \geq 1$, then $\Delta \leq 0$. We get that for any $\beta \geq 1$, 
\begin{eqnarray}
d^{\lambda,1,\beta}_M (\ve{x},\ve{z}) & \leq & d^{\lambda,1,\beta}_M (\ve{x},\ve{y}) + d^{\lambda,1,\beta}_M (\ve{y},\ve{z})\:\:,
\end{eqnarray}
and $d^{\lambda,1,\beta}_M$ satisfies the triangle inequality. For $\beta = 1/2$, it is trivial to check that for any $\ve{x} \in \bigtriangleup_n$, the (no) transportation plan $P = \mathrm{Diag}(\ve{x})$ is in $U(\ve{x},\ve{x})$ and satisfies
\begin{eqnarray}
\lefteqn{\inner{P}{M} - \frac{1}{\lambda}\cdot \left(H_q(P) - \frac{1}{2}\cdot( H_q(\ve{x}) + H_q(\ve{x}))\right)}\nonumber\\
 & = & 0 - \frac{1}{\lambda}\cdot \left(H_q(\ve{x}) - H_q(\ve{x}) \right) = 0\:\:.
\end{eqnarray}
This ends the proof of Theorem \ref{thdist2}.
\end{proof}
Notice that Theorem \ref{thdist3} is in fact a direct consequence of Theorem \ref{thdist2}. To finish up, we now prove Theorem \ref{thAxiom}. To simplify notations, let
\begin{eqnarray}
U_\alpha(\ve{r}, \ve{c}) & \defeq & \left\{ 
P\in U(\ve{r},\ve{c}): 
H_q(P) - H_q(\ve{r}) - H_q(\ve{c}) \geq \alpha(\lambda) \right\}\:\:.\label{defUalpha}
\end{eqnarray}
Suppose $P, Q$ in Lemma \ref{lemglu} are such that $P, Q \in U_\lambda(\ve{x}, \ve{y})$. In this case, 
\begin{eqnarray}
H_q(P) - H_q(\ve{x}) - H_q(\ve{y}) & \geq & \alpha
\end{eqnarray}
and so point 2. in Lemma \ref{lemglu} brings
\begin{eqnarray}
H_q(S) - H_q(\ve{x}) - H_q(\ve{z}) & \geq & \alpha\:\:,
\end{eqnarray}
so $S \in U_\lambda(\ve{x}, \ve{z})$. The proof of \cite[Theorem 1]{cSD} can then be used to show that $\forall \ve{x}, \ve{y}, \ve{z} \in \bigtriangleup_n$,
\begin{eqnarray}
d_{M, \alpha, q}(\ve{x},\ve{z}) & \leq & d_{M, \alpha, q}(\ve{x},\ve{y})+d_{M, \alpha, q}(\ve{y},\ve{z})\:\:.\label{defdist22}
\end{eqnarray}
It is easy to check that $d_{M, \alpha, q}$ is non negative and that $\mathbbm{1}_{\lbrace \ve{r}=\ve{c}\rbrace} d_{M, \alpha, q}(\ve{r},\ve{c})$ meets, in addition, the identity of the indiscernibles. This achieves the proof of Theorem \ref{thAxiom}.

\section{Proof of Theorem \ref{thSortrot}}\label{proof_thSortrot}

\noindent \textbf{Basic facts and definitions} --- In this proof, we make two simplifying assumptions: (i) we consider matrices either as matrices or as vectorized matrices without ambiguity, and (ii) we let $\phi(P) \defeq -H_q(P)$, noting that the domain of $\phi$ is $\bigtriangleup_{n^2}$ (nonnegative matrices with row- and column-sums in the simplex) when $P \in U(\ve{r},\ve{c})$. Since $\phi$ is convex, we can define a \textit{Bregman divergence} with generator $D_\phi$ \cite{bTR} as:
\begin{eqnarray}
D_\phi(P\|R) & \defeq & \phi(P) - \phi(R) - \inner{\nabla\phi(R)}{P-R}\:\:.\nonumber
\end{eqnarray}
We define 
\begin{eqnarray}
a_{ij} & \defeq & \alpha_i + \lambda m_{ij} + \beta_j\:\:,\label{defaij}
\end{eqnarray}
so that 
\begin{eqnarray}
p_{ij}&  = & \exp_q(-1)\exp_q^{-1}(a_{ij})
\end{eqnarray} 
in eq. (\ref{eq:KKT}). 
Finally, let us denote for short 
\begin{eqnarray}
D_q(P\|R) & \defeq & K_{1/q}(P^q, R^q)\:\:,\label{defDK}
\end{eqnarray}
so that we can, reformulate eq. (\ref{tsent}) as:
\begin{eqnarray}
d^{\lambda,q}_M (\ve{r},\ve{c}) & = & \frac{1}{\lambda}\cdot \min_{P \in U(\ve{r},\ve{c})} D_q(P \| \tilde{U}) + g(M)\:\:,\label{tsent22}
\end{eqnarray}
and our objective "reduces" to the minimization of $D_q(P \| \tilde{U})$ over $U(\ve{r},\ve{c})$. In \sotrot~(Algorithm \ref{alg:SecondOrderSinkhorn}), we just care for a single constraint out of the two possible in $U(\ve{r},\ve{c})$, so we will focus without loss of generality on the row constraint and therefore to the solution of:
\begin{eqnarray}
P^\star & \defeq & \arg \min_{P \in \mathbb{R}_+^{n\times n} : P\ve{1} = \ve{r}} D_q(P \| \tilde{U}) \:\:.\label{tsent22p}
\end{eqnarray}
The same result would apply to the column constraint.

\noindent \textbf{Convergence proof} --- We reuse the theory of \textit{auxiliary functions} developed for the iterative constrained minimization of Bregman divergences \cite{bTR,ddlIF}. We reuse notation "$\diamond$" following \cite{cssLRj,nnOT} and define for any $\ve{y} \in \mathbb{R}^{n}, P\in {\mathbb{R}}^{n\times n}$ matrix $\ve{y} \diamond_q P \in \mathbb{R}^{n\times n}$ such that
\begin{eqnarray}
\lefteqn{(\ve{y} \diamond_q P)_{ij}}\nonumber\\
 & \defeq & \frac{\exp_q^{-1}(y_i)p_{ij}}{\exp_q\left[(1-q)y_i\exp_q^{1-q}(y_i)\log_q(p_{ij})\right]}\label{defdiam}\:\:.
\end{eqnarray}
We also define key matrix $\tilde{P}\in \mathbb{R}^{n\times n}$ with:
\begin{eqnarray}
\tilde{P} & \defeq & \ve{r}\ve{c}^\top\label{defPtilde}\:\:.
\end{eqnarray} 
Let us denote 
\begin{eqnarray*}
\mathcal{Q} & \defeq & \left\{Q \in \mathbb{R}^{n\times n} : 
Q = \exp_q(-1) \exp_q^{-1}(\ve{\alpha}^\top\ve{1} + \lambda M + \ve{1}^\top \ve{\beta})\right\}\:\:.\\
\mathcal{P} & \defeq & \{P\in \bigtriangleup_{n^2} : P\ve{1} = \tilde{P}\ve{1} = \ve{r}\}\:\:.
\end{eqnarray*}
One function will be key.
\begin{definition}\label{defaux}
We define $A(P,\ve{y}) \defeq \sum_i A_i(P,\ve{y})$, with:
\begin{eqnarray}
\lefteqn{A_i(P,\ve{y})}\nonumber\\
 & \defeq & y_i r_i + \sum_j (p^q_{ij} - \exp_q^q(-1)\exp_q^{-q}(a_{ij} - y_i))\label{af2}\:\:.
\end{eqnarray}
Here $a_{ij}$ is defined in eq. (\ref{defaij}), $r_i$ is the $i$-th coordinate in $\ve{r}$ (the row marginal constraint), and $\ve{y}\in \mathbb{R}^n$.
\end{definition} 
\begin{lemma}\label{lemaux}
For any $\ve{y}$, 
\begin{eqnarray}
A(P,\ve{y}) & = & D_\phi (\tilde{P} \| P) - D_\phi (\tilde{P} \| \ve{y}\diamond_q P)\:\:.
\end{eqnarray}
Furthermore, $A(P,\ve{0}) = 0$.
\end{lemma}
\begin{proof}
We have
\begin{eqnarray}
\lefteqn{D_\phi(\tilde{P} \| P) - D_\phi(\tilde{P} \| \ve{y}\diamond_q P)}\nonumber\\
 & = & -D_\phi(P \| \ve{y}\diamond_q P) \nonumber\\
 & & + \inner{\tilde{P}-P}{\nabla\phi(\ve{y}\diamond_q P) - \nabla\phi(P)}\nonumber\:\:.
\end{eqnarray}
\end{proof}
Because a Bregman divergence is non-negative and $A(P,\ve{0}) = 0$, if, as long as there exists some $\ve{y}$ for which $A(P,\ve{y}) > 0$ we keep on updating $P$ by replacing it by $\ve{y}^*\diamond_q P$ such that $A(P,\ve{y}^*) > 0$, then the sequence
\begin{eqnarray}
P_0 = \tilde{U} \rightarrow P_1 \defeq \ve{y}^*_0\diamond_q P_0 \rightarrow P_2 \defeq \ve{y}^*_1\diamond_q P_1 \cdots
\end{eqnarray}
will converge to a limit matrix in the sequence,
\begin{eqnarray}
\lim_j P_{j} \defeq \ve{y}^*_{j-1}\diamond_q P_{j-1}\:\:.
\end{eqnarray}
This matrix turns out to be the one we seek.
\begin{theorem}\label{thconv}
Let $P_{j+1} \defeq \ve{y}_{j}\diamond_q P_{j}$ (with $P_0 \defeq \tilde{U}$) be such that $A(P_j,\ve{y}_j) > 0, \forall j\geq 0$, and the sequence ends when no such $\ve{y}_j$ exists. Then $\mathcal{S} \defeq \{P_j\}_{j\geq 0} \subset \bar{\mathcal{Q}}$. If furthermore $\mathcal{S}$ lies in a compact of $\bar{\mathcal{Q}}$, then it satisfies
\begin{eqnarray}
P^\star \defeq \lim_j P_j & = & \arg\min_{P \in \mathcal{P}} D_q(P \| \tilde{U})\:\:.\label{defPINF}
\end{eqnarray}
\end{theorem}
\begin{myproofsketch}
The proof relies on two steps, first that
\begin{eqnarray}
P^\star \defeq \lim_j P_j & = & \arg\min_{P \in \mathcal{P}} D_\phi(P \| \tilde{U})\:\:,
\end{eqnarray}
and then the fact that (\ref{defPINF}) holds as well, which "amounts" to replacing $D_\phi$, which is Bregman, by $D_q$, which is \textit{not}. Because it is standard in Bregman divergences, we sketch the first step. The fundamental result we use is adapted from \cite{ddlIF} (see also \cite[Theorem 1]{cssLRj}).
\begin{theorem}
Suppose that $D_\phi(\tilde{P},\tilde{U}) < \infty$. Then there exists a unique $P^\star$ satisfying the following four properties:
\begin{enumerate}
\item $P^\star \in \mathcal{P}\cap\bar{\mathcal{Q}}$
\item $\forall P\in\mathcal{P},\forall R\in\bar{\mathcal{Q}}, D_\phi(P\|R) = D_\phi(P\|P^\star) + D_\phi(P^\star\|R)$
\item $P^\star = \underset{P \in \mathcal{P}}{\arg\min} 	D_\phi(P\|\tilde{U})$
\item $P^\star = \underset{R \in \bar{\mathcal{Q}}}{\arg\min} D_\phi(\tilde{P}\|R)$
\end{enumerate}
Moreover, any of these four properties determines $P^\star$ uniquely.
\end{theorem}
It is not hard to check that $\tilde{U}\in \bar{\mathcal{Q}}$ and whenever $P_j\in \bar{\mathcal{Q}}$, then $\ve{y} \diamond_q P_{j} \in \bar{\mathcal{Q}}, \forall \ve{y}$, so we indeed have $\mathcal{S} \subset \bar{\mathcal{Q}}$. With the constraint that $A(P_j,\ve{y}_j) > 0, \forall j\geq 0$, it follows from Lemma \ref{lemaux} that $A(P,\ve{y})$ is an auxiliary function for $\mathcal{S}$ \cite{cssLRj} \textit{if} we can show in addition that if $\ve{y} = \ve{0}$ is a maximum of $A(P,\ve{y})$, then $P\in \mathcal{P}$. To remark that this is true, we have 
\begin{eqnarray}
\nabla A(P,\ve{y})_{\ve{y}} & = & \ve{r} - P\ve{1}\:\:,
\end{eqnarray}
so whenever $A(P,\ve{y})$ reaches a maximum in $\ve{y}$, we indeed have $P\ve{1} = \ve{r}$ and so $P\in \mathcal{P}$, and if $\ve{y} = \ve{0}$ then because a Bregman divergence satisfies the identity of the indiscernibles, if $\ve{y} = \ve{0}$ is the maximum, then $\mathcal{S}$ has converged to some $P^\star$. From 4. above, we get
\begin{eqnarray}
P^\star & = & \underset{R \in \bar{\mathcal{Q}}}{\arg\min} D_\phi(\tilde{P}\|R)\:\:,
\end{eqnarray}
and so from 3. above, we also get
\begin{eqnarray}
P^\star & = & \underset{P \in \mathcal{P}}{\arg\min} 	D_\phi(P\|\tilde{U}) \:\:.
\end{eqnarray}
To "transfer" this result to $D_q$, we just need to remark that there is one remarkable trivial equality:
\begin{eqnarray}
D_\phi(P\|R) & = & D_q(P\|R)  - \sum_{i,j}(p^q_{ij} - r^q_{ij})\label{eqbK}\:\:,
\end{eqnarray}
so that even when $K_{1/q}$ is \textit{not} a Bregman divergence for a general $q$, it still meets the \textit{Bregman triangle equality} \cite{aIG}.
\begin{lemma}
We have;
\begin{eqnarray}
\lefteqn{D_q(P\|R) + D_q(R\|S) - D_q(P\|S)}\nonumber\\
 &= & D_\phi(P\|R) + D_\phi(R\|S) - D_\phi(P\|S)\nonumber\\
&=& \inner{P-R}{\nabla\phi(S) - \nabla\phi(R)}\label{eq:breg_triangle}\:\:.
\end{eqnarray}
\end{lemma}
Hence, point 2. implies as well
\begin{eqnarray}
D_q(P\|R) & = & D_q(P\|P^\star) + D_q(P^\star\|R) \:\:,
\end{eqnarray}
$\forall P\in\mathcal{P},\forall R\in\bar{\mathcal{Q}}$, and so $D_q(P\|\tilde{U}) = D_q(P\|P^\star) + D_q(P^\star\|\tilde{U}), \forall P\in\mathcal{P}$, so that we also have (since $D_q$ is non negative and satisfies $D_q(P\|P) = 0$)
\begin{eqnarray*}
P^\star & = &  \arg\min_{P \in \mathcal{P}} D_q(P \| \tilde{U})\:\:,
\end{eqnarray*}
as claimed (end of the proof of Theorem \ref{thconv}).
\end{myproofsketch}\\
Figure \ref{f-aux} summarizes Theorem \ref{thconv}.
\begin{figure}[t]
\begin{center}
\includegraphics[trim=0bp 0bp 0bp 0bp,clip,width=0.5\columnwidth]{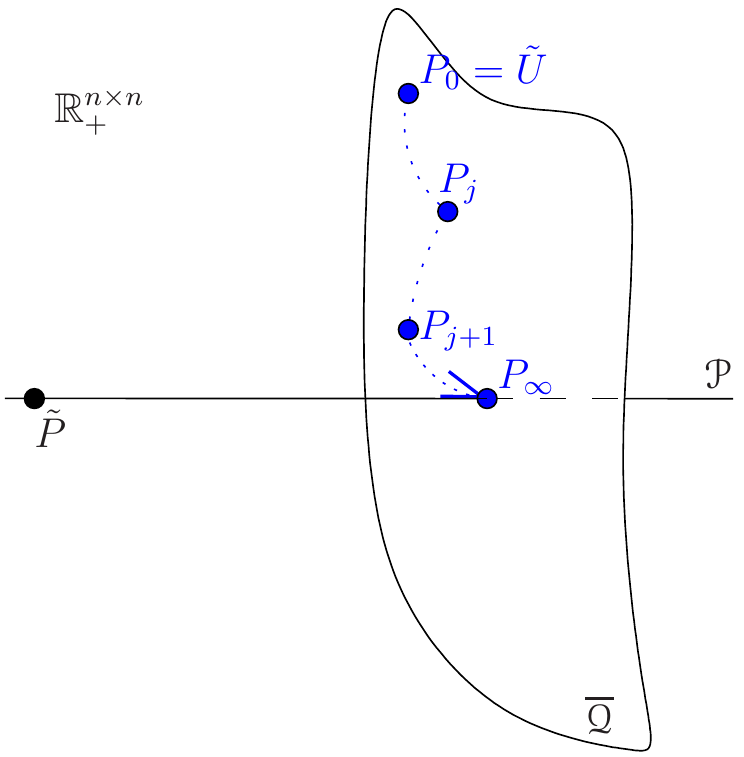}
\end{center}
\caption{High level overview of the proof of Theorem \ref{thconv} (see text for details).}
\label{f-aux}
\end{figure}
We are left with the problem of finding an auxiliary function for the sequence $\mathcal{S}$, which we recall boils down to finding, whenever it exists, some $\ve{y}$ such that $A(P,\ve{y}) > 0$. 
\begin{theorem}\label{thAUX2}
$A(P,\ve{y})$ is an auxiliary function for $\mathcal{S}$ for the sequence of updates $\ve{y}$ given as in steps 6-11 of \sotrot~(Algorithm \ref{alg:SecondOrderSinkhorn}).
\end{theorem}
\begin{proof}
We shall need the complete Taylor expansion of $A(P,\ve{y})$.
\begin{lemma}
Let us denote for short $\gamma \defeq 1-q$. The Taylor series expansion of $A_i(P,\ve{y})$ (as defined in Definition \ref{defaux}) is:
\begin{eqnarray}
\lefteqn{A_i(P,\ve{y})}\nonumber\\
 & = & y_i(r_i - \sum_jp_{ij}) \nonumber\\
 & & - \sum_j p_{ij}\sum_{k = 2}^{\infty} \left[\frac{1}{k}\prod_{l=1}^{k-1}(\gamma + q/l)\right]y_i^k\left(\frac{p_{ij}^\gamma}{q}\right)^{k-1}\:\:.
\end{eqnarray}
\end{lemma}
\begin{proof}
Let us denote $f(x) = \exp_q^{-q}(x)$. We have: 
\begin{eqnarray}
\frac{\mathrm{d}}{\mathrm{d}x}f(x) &=&  q \exp_q^{1-q}(x)\frac{\mathrm{d}}{\mathrm{d}x}\exp_q^{-1}(x)\nonumber\\
&= & -q \exp_q^{-1}(x)\:\:.
\end{eqnarray}
A simple recursion also shows ($\forall k\geq 2$):
\begin{eqnarray}\label{eq:derivatives}
\lefteqn{\frac{\mathrm{d}^k}{\mathrm{d}x^k}\exp_q^{-1}(x)}\nonumber\\
 & =&  (-1)^k\left[\prod_{i=1}^k(i-(i-1)q)\right] \exp_q^{kq - (k+1)}(x)\:\:,\nonumber
\end{eqnarray}
which yields $\forall k\geq 1$,
\begin{eqnarray}
\frac{\mathrm{d}^k}{\mathrm{d}x^k} f(x)&= & -q \frac{\mathrm{d}^{k-1}}{\mathrm{d} x^{k-1}}\exp_q^{-1}(x) \nonumber\\
&= & (-1)^kq\left[\prod_{i=1}^{k-1}(i\gamma + q)\right] \exp_q^{-(k-1)\gamma -1}(x) \nonumber\:\:.
\end{eqnarray}
Since $\exp_q^q(-1) = \exp_q(-1)/q$ and $\forall i,j, p_{ij} = \exp_q(-1)\exp_q^{-1}(a_{ij})$, writing the Taylor development of $f$ at point $a_{ij}$ evaluated at $y_i$, and adding the $y_ir_i + \sum_jp_{ij}^q$ term, we obtain the desired result.
\end{proof}
We have two special reals to define, $t_i$ and $z_i$. If $r_i \leq \sum_j p_{ij}$, we let $t_i$ denote the maximum of the second order approximation of $A_i(P,\ve{y})$, 
\begin{eqnarray}
T^{(2)}_i(y_i) & \defeq & y_i(r_i - \sum_j p_{ij}) - \frac{y_i^2}{2}\sum_j\frac{p_{ij}^{1+\gamma}}{q}\:\:,
\end{eqnarray}
\textit{i.e.} the root of 
\begin{eqnarray*}
\frac{\mathrm{d}}{\mathrm{d} y}T^{(2)}(y_i) & = & (r_i - \sum_j p_{ij}) - y_i\sum_j\frac{p_{ij}^{1+\gamma}}{q}\:\:.
\end{eqnarray*}
If $\sum_j p_{ij} \leq r_i$, we let $z_i$ be the the largest root of
\begin{eqnarray}
R_i & \defeq & (r_i - \sum_j p_{ij}) \nonumber\\
 & & - y_i\sum_j\frac{p_{ij}^{1+\gamma}}{q} - y_i^2(2-q)\sum_j\frac{p_{ij}^{1+2\gamma}}{q^2}\:\:.\label{defRI}
\end{eqnarray}
We shall see that $z_i$ is positive. Let $y^*_i \defeq t_i$ if $r_i \leq \sum_j p_{ij}$, and $y^*_i \defeq z_i$ otherwise. We first make the assumption that
\begin{eqnarray}
\left| \frac{y^*_i p_{ij}^\gamma}{q}\cdot \left(\gamma + \frac{q}{3}\right)\right| & \leq & \frac{1}{2}\:\:, \forall i,j\:\:.\label{ineq2Y}
\end{eqnarray}
Under this assumption, we have two cases.\\
\noindent ($\star$) Case $r_i \leq \sum_j p_{ij}$. 
By definition, we have in this case that $y_i = t_i \leq 0$ in \sotrot~(Step 10). We also have
\begin{eqnarray}
\lefteqn{A_i(P,\ve{y})}\nonumber\\
 &= & T^{(2)}(y_i) \nonumber\\
 & & - \underbrace{\sum_j p_{ij}\sum_{k = 3}^{\infty} \left[\frac{1}{k}\prod_{l=1}^{k-1}(\gamma + q/l)\right]y_i^k\left(\frac{p_{ij}^\gamma}{q}\right)^{k-1}}_{\defeq S_3}\:\:.
\end{eqnarray}
Since $y_i = t_i \leq 0$, $S_3$ is an alternating series, that is a series whose general term is alternatively positive and negative. Under assumption (\ref{ineq2Y}), the module of its general term is decreasing. A classic result on series allows us to deduce from this fact that (a) $S_3\ll \infty$ and (b) the sign of $S_3$ is that of its first term, \textit{i.e.}, it is negative. Since $A_i(P,\ve{y}) = T^{(2)}(y_i) - S_3$, we have that 
\begin{eqnarray}
A_i(P,\ve{y})&  \geq & T^{(2)}(y_i)  = 0\:\:.
\end{eqnarray}
Note also that $A_i(P,\ve{y}) = 0$ iff $\sum_j p_{ij} = r_i$ as $T^{(2)}(y_i)$ is decreasing on $[t_i,0]$ and $T^{(2)}(0) = 0$. Hence, for the choice in Step 10, $A_i(P,\ve{y})$ is an auxiliary function for variable $i$.\\

\noindent ($\star$) Case $\sum_j p_{ij} \leq r_i$: we still have $A_i(P,\ve{y}) = T^{(2)}(y_i) - S_3$, but this time $y_i$ will be positive, ensuring $y_i(r_i - \sum_j p_{ij})\geq 0$. We first show that $S_3$ is upperbounded by a geometric series under assumption (\ref{ineq2Y}):
\begin{eqnarray*}
\lefteqn{S_3}\nonumber\\
  &= & \sum_jp_{ij}y_i^{3}\left(\frac{p_{ij}^\gamma}{q}\right)^2\sum_{k=0}^\infty \frac{y_i^k}{k+3}\left[\prod_{l=1}^{k+2}(\gamma + q/l)\right]\left(\frac{p_{ij}^\gamma}{q}\right)^k\\
&\leq  & \sum_jp_{ij}(1-q/2)\frac{y_i^{3}}{3}\left(\frac{p_{ij}^\gamma}{q}\right)^2\sum_{k=0}^\infty \left(\frac{y_ip_{ij}^\gamma}{q}(\gamma + q/3)\right)^k\\
&=& \sum_jp_{ij}(1-q/2)\frac{y_i^{3}}{3}\left(\frac{p_{ij}^\gamma}{q}\right)^2 \times \frac{1}{1-\frac{y_ip_{ij}^\gamma}{q}(\gamma + q/3)}\\
&\leq & (2-q)\sum_j p_{ij} \frac{y_i^3}{3}\left(\frac{p_{ij}^\gamma}{q}\right)^2\:\:,
\end{eqnarray*}
which conveniently yields
\begin{eqnarray}\label{eq:lower_bound}
A_i(P,\ve{y}) & \geq & T^{(2)}(y_i) - (2-q)\sum_j p_{ij} \frac{y_i^3}{3}\left(\frac{p_{ij}^\gamma}{q}\right)^2\:\:.
\end{eqnarray}
The derivative of the right-hand term of (\ref{eq:lower_bound}) is $R_i$ defined in eq. (\ref{defRI}) above.
Let us define:
\begin{eqnarray}
a &\defeq & (2-q)\sum_j\frac{p_{ij}^{1+2\gamma}}{q^2}\:\:,\\
b &\defeq & \sum_j\frac{p_{ij}^{1+\gamma}}{q}\:\:,\\
c &\defeq & -(r_i - \sum_j p_{ij})\:\:.
\end{eqnarray}
We have $ac <0$ and consequently the discriminant $\Delta \defeq b^2 - 4ac > b^2$, implying $R_i$ has a positive root $z_i \defeq (-b+\sqrt{\Delta})/(2a)$ which maximises the right-hand term of \ref{eq:lower_bound}, and is such that this right-hand term is positive. Further, we again have that $z_i = 0$ iff $\sum_j p_{ij} = r_i$. It is easy to check that $z_i = y_i$ in Step 8 of \sotrot, for which we check that $A_i(P,\ve{y})\geq 0$, wich equality iff $\sum_j p_{ij} = r_i$.
Hence, for the choice in Step 8, $A_i(P,\ve{y})$ is an auxiliary function for variable $i$.\\

\noindent We can now conclude that under assumption (\ref{ineq2Y}), $A(P,\ve{y})$ is an auxiliary function.\\

\noindent If assumption (\ref{ineq2Y}) does not hold, then notice that this cannot not hold at convergence for coordinate $i$. For this reason, $r_i \neq \sum_j p_{ij}$ and the sign $\mathrm{sign}(r_i - \sum_j p_{ij})$ is also well defined. Therefore, we just need to pick a value for $y_i \neq 0$ which guarantees $A_i(P,\ve{y}) > 0$. To do so, we pick 
\begin{eqnarray}
y_i & = & \frac{q \cdot \mathrm{sign}(r_i - \sum_j p_{ij})}{(6-4q)\cdot \max_j p_{ij}^{1-q}}\:\:,
\end{eqnarray}
remarking that this $y_i$ indeed violates (\ref{ineq2Y}) (recalling $\gamma \defeq 1-q$). We also have $|y_i| \in (0, n^{2(1-q)}/2]$. Notice that this choice guarantees $A_i(P,\ve{y}) > 0$. (end of the proof of Theorem \ref{thAUX2})
\end{proof}
Theorems \ref{thconv} and \ref{thAUX2} altogether prove Theorem \ref{thSortrot}.

\section*{Supplementary Material: experiments}\label{exp_expes}

\section{Per county error distribution, \trot~survey vs Florida average}\label{exp_expFAV}

Figure \ref{fig:error-dist} displays the empirical distribution of the errors for \trot~vs Florida average. While not being a true distribution of the solution error of \trot~--- in a Bayesian sense ---, the graph should convey the intuition that algorithms with a distribution that shrinks around zero provide better inference.

\begin{figure}[!ht]
  \centering
    \includegraphics[width=0.45\textwidth]{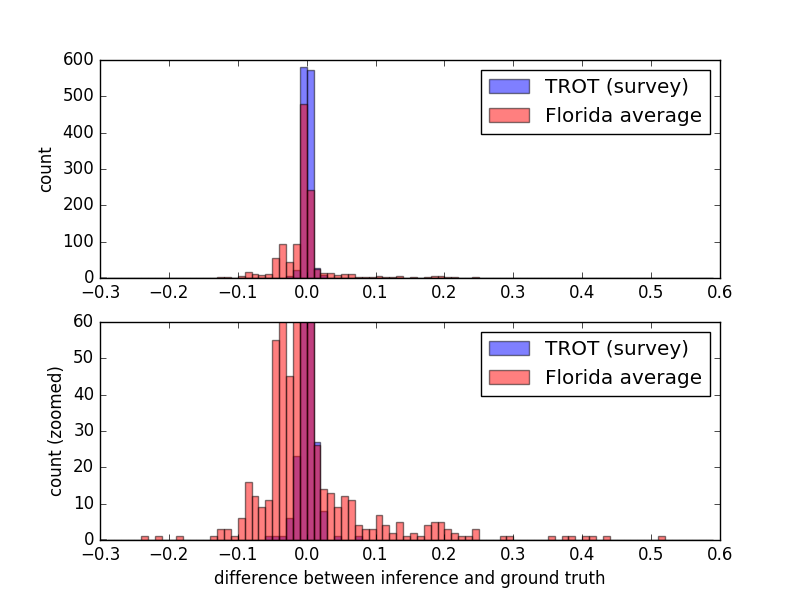}
        \caption{(Signed) error distribution of \trot~compared to Florida-average.}
        \label{fig:error-dist}
\end{figure}

\section{Per county errors, \trot~survey vs \trot~$\ve{1}\ve{1}^\top$}\label{exp_expSURV}

Figure \ref{fig:bars} confronts the prediction errors by county of \trot~when we use $M=\msur$ (survey) and $M = \mnop (= \bm{1}\bm{1}^\top)$ as cost matrix: while the overall performance of the two algorithms is very close, the graph demonstrates that \trot~optimized with $\msur$ achieves very often smaller error, although the average error is worsen by few particularly bad counties. 

\begin{figure*}[!ht]
  \centering
    \includegraphics[width=0.95\textwidth]{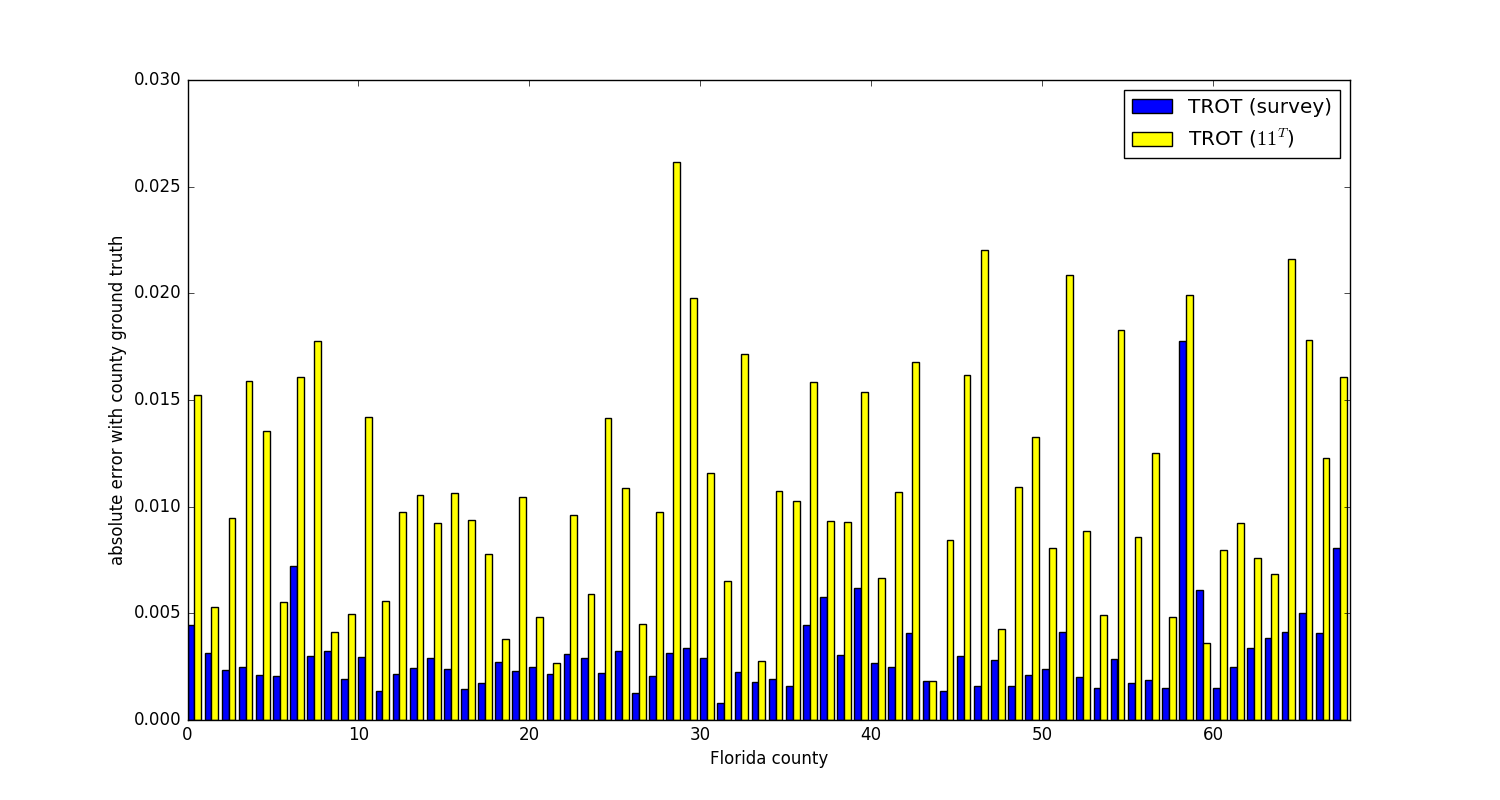}
        \caption{Absolute error of \trot~optimized with $M$ compared to with no prior.}
        \label{fig:bars}
\end{figure*}

\end{document}